\documentclass[11pt]{article}
\pdfoutput=1

\usepackage[utf8]{inputenc}
\usepackage[T1]{fontenc}
\usepackage[english]{babel}
\usepackage{microtype}
\usepackage[in]{fullpage}
\usepackage{graphicx}
\usepackage{subfigure}
\usepackage{amsmath,amssymb,amsthm}
\usepackage{mathtools}
\usepackage{mathrsfs}
\usepackage{thmtools}
\usepackage{thm-restate}
\usepackage{xfrac}
\usepackage{comment}
\usepackage{enumitem}
\usepackage{xr}
\usepackage[style=numeric,natbib=true,maxbibnames=10000,maxalphanames=5,backend=biber,sortcites=true]{biblatex}
\bibliography{biblio}
\usepackage{etoolbox}
\usepackage{hyperref}

\DeclareFieldFormat{url}{\url{#1}}
\DeclareFieldFormat[unpublished]{note}{\mkbibemph{#1}}
\DeclareFieldFormat{year}{\mkbibemph{#1}}

\DeclareBibliographyDriver{unpublished}{
\usebibmacro{bibindex}
\usebibmacro{begentry}
\usebibmacro{author/translator+others}
\setunit{\labelnamepunct}\newblock
\usebibmacro{title}
\newunit\newblock
\printfield{note}
\setunit{\addcomma\space}
\usebibmacro{date}
\newunit\newblock
\usebibmacro{url}
\usebibmacro{finentry}}

\DeclareBibliographyDriver{article}{
\usebibmacro{bibindex}
\usebibmacro{begentry}
\usebibmacro{author/translator+others}
\setunit{\labelnamepunct}\newblock
\usebibmacro{title}
\newunit\newblock
\usebibmacro{journal}
\setunit{\addcomma\space}
\usebibmacro{date}
\newunit\newblock
\usebibmacro{url}
\usebibmacro{finentry}}

\DeclareBibliographyDriver{inproceedings}{
\usebibmacro{bibindex}
\usebibmacro{begentry}
\usebibmacro{author/translator+others}
\setunit{\labelnamepunct}\newblock
\usebibmacro{title}
\newunit\newblock
\usebibmacro{booktitle}
\setunit{\addcomma\space}
\usebibmacro{date}
\newunit\newblock
\usebibmacro{url}
\usebibmacro{finentry}}

\newtoggle{isicml}
\togglefalse{isicml}

\usepackage{pgfplots}
\usepackage{mathtools}
\usepackage[utf8]{inputenc}
\usepackage[T1]{fontenc}
\usepackage[english]{babel}
\usepackage{amsmath,amssymb,amsthm}
\usepackage{mathtools}
\usepackage{mathrsfs}
\usepackage{thmtools}
\usepackage{xfrac}

\declaretheorem{theorem}

\declaretheorem[sibling=theorem]{fact}
\declaretheorem[sibling=theorem]{lemma}
\declaretheorem[sibling=theorem]{corollary}
\declaretheorem[sibling=theorem]{definition}

\newcommand{\R}{\ensuremath{\mathbb{R}}}
\newcommand{\E}{\ensuremath{\mathbb{E}}}
\newcommand{\I}{\ensuremath{\mathbb{I}}}
\newcommand{\Eop}{\ensuremath{\mathop{\mathbb{E}}}}
\newcommand{\prob}{\ensuremath{\mathbb{P}}}
\newcommand{\probop}{\ensuremath{\mathop{\mathbb{P}}}}
\newcommand{\eps}{\ensuremath{\varepsilon}}
\newcommand{\vz}{\ensuremath{z}}
\newcommand{\vx}{\ensuremath{x}}
\newcommand{\vzbar}{\overline{\ensuremath{z}}}
\newcommand{\zbar}{\overline{\ensuremath{z}}}
\newcommand{\vg}{\ensuremath{g}}
\newcommand{\vgbar}{\overline{\ensuremath{g}}}
\newcommand{\vw}{\ensuremath{w}}
\newcommand{\vmu}{\ensuremath{\mu}}
\newcommand{\what}{\ensuremath{\widehat{w}}}
\newcommand{\matI}{\ensuremath{I}}
\newcommand{\thetastar}{\ensuremath{\theta^\star}}

\newcommand{\normal}{\ensuremath{\mathcal{N}}}
\newcommand{\M}{\ensuremath{\mathcal{M}}}
\newcommand{\rademacher}{\ensuremath{\mathcal{R}}}
\newcommand{\distP}{\ensuremath{\mathcal{P}}}
\newcommand{\powerset}{\ensuremath{\mathscr{P}}}
\newcommand{\perturbB}{\ensuremath{\mathcal{B}}}
\newcommand{\defeq}{\ensuremath{=}}

\newcommand{\eqdef}{\defeq}
\newcommand{\sign}{\ensuremath{\text{sgn}}}

\newcommand{\pmset}{\ensuremath{\{\pm 1\}}}

\newcommand\numberthis{\addtocounter{equation}{1}\tag{\theequation}}

\newcommand{\loss}{\ensuremath{\mathcal{L}}}

\DeclarePairedDelimiter{\norm}{\lVert}{\rVert}
\DeclarePairedDelimiter{\abs}{\lvert}{\rvert}
\DeclarePairedDelimiter{\parens}{\lparen}{\rparen}
\DeclarePairedDelimiter{\brackets}{[}{]}
\DeclarePairedDelimiter{\ip}{\langle}{\rangle}

\usepackage[textsize=scriptsize,textwidth=2cm]{todonotes}

\newcommand\AND{
\end{tabular}\hfil\linebreak[4]\hfil
\begin{tabular}[t]{c}\ignorespaces
}

\title{Adversarially Robust Generalization Requires More Data}
\author{Ludwig Schmidt\\ MIT \and Shibani Santurkar\\ MIT \and Dimitris Tsipras\\ MIT \AND Kunal Talwar\\ Google Brain \and Aleksander M\k{a}dry\\ MIT}
\date{}

\begin{document}
\maketitle

\begin{abstract}
  Machine learning models are often susceptible to adversarial perturbations of their inputs.
Even small perturbations can cause state-of-the-art classifiers with high ``standard'' accuracy to produce an incorrect prediction with high confidence.
To better understand this phenomenon, we study adversarially robust learning from the viewpoint of generalization.
We show that already in a simple natural data model, the sample complexity of robust learning can be significantly larger than that of ``standard'' learning.
This gap is information theoretic and holds irrespective of the training algorithm or the model family.
We complement our theoretical results with experiments on popular image classification datasets and show that a similar gap exists here as well.
We postulate that the difficulty of training robust classifiers stems, at least partially, from this inherently larger sample complexity.

\end{abstract}

\section{Introduction}
Modern machine learning models achieve high accuracy on a broad range of datasets, yet can easily be misled by small perturbations of their input.
While such perturbations are often simple noise to a human or even imperceptible, they cause state-of-the-art models to misclassify their input with high confidence.
This phenomenon has first been studied in the context of secure machine learning for spam filters and malware classification~\cite{DDMSV04,LM05,biggio2017wild}.
More recently, researchers have demonstrated the phenomenon under the name of \emph{adversarial examples} in image classification~\cite{SzegedyZSBEGF13,GoodfellowSS14}, question answering~\cite{jia2017adversarial}, voice recognition~\cite{CMVZSSWZ16,ZYJZZX17,SM17,carlini2018audio}, and other domains (for instance, see \cite{GPMBM16,CANK17,AMT17,BM17,HPGDA17,KFS17,XCLRDS17,HAHO18}).
Overall, the existence of such adversarial examples raises concerns about the robustness of trained classifiers.
As we increasingly deploy machine learning systems in safety- and security-critical environments, it is crucial to understand the robustness properties of our models in more detail.

A growing body of work is exploring this robustness question from the security perspective by proposing \emph{attacks} (methods for crafting adversarial examples) and \emph{defenses} (methods for making classifiers robust to such perturbations).
Often, the focus is on deep neural networks, e.g., see~\cite{SharifBBR16,MoosDez16,papernot2016distillation,CarliniW16a,TramerKPBM17,madry2017towards,xu2017feature,he2017weak}.
While there has been success with robust classifiers on simple datasets~\cite{madry2017towards,kolter2017provable,sinha2017certifiable,raghunathan2018certified}, more complicated datasets still exhibit a large gap between ``standard'' and robust accuracy~\cite{CarliniW16a,athalye2018obfuscated}.
An implicit assumption underlying most of this work is that the same training dataset that enables good standard accuracy also suffices to train a robust model.
However, it is unclear if this assumption is valid. 

So far, the \emph{generalization} aspects of adversarially robust classification have not been thoroughly investigated.
Since adversarial robustness is a learning problem, the statistical perspective is of integral importance.
A key observation is that adversarial examples are not at odds with the standard notion of generalization as long as they occupy only a small total measure under the data distribution.
So to achieve adversarial robustness, a classifier must generalize in a stronger sense.
We currently do not have a good understanding of how such a stronger notion of generalization compares to standard ``benign'' generalization, i.e., without an adversary.

In this work, we address this gap and explore the statistical foundations of adversarially robust generalization.
We focus on sample complexity as a natural starting point since it underlies the core question of when it is possible to learn an adversarially robust classifier.
Concretely, we pose the following question:
\begin{quote}
    \emph{How does the sample complexity of standard generalization compare to that of adversarially robust generalization?}
\end{quote}

To study this question, we analyze robust generalization in two distributional models.
By focusing on specific distributions, we can establish information-theoretic lower bounds and describe the exact sample complexity requirements for generalization.
We find that even for a simple data distribution such as a mixture of two class-conditional Gaussians, the sample complexity of robust generalization is significantly larger than that of standard generalization.
Our lower bound holds for \emph{any} model and learning algorithm.
Hence no amount of algorithmic ingenuity is able to overcome this limitation.

In spite of this negative result, simple datasets such as MNIST have recently seen significant progress in terms of adversarial robustness~\cite{madry2017towards,kolter2017provable,sinha2017certifiable,raghunathan2018certified}.
The most robust models achieve accuracy around 90\% against large $\ell_\infty$-perturbations.
To better understand this discrepancy with our first theoretical result, we also study a second distributional model with binary features.
This binary data model has the same standard generalization behavior as the previous Gaussian model.
Moreover, it also suffers from a significantly increased sample complexity whenever one employs \emph{linear} classifiers to achieve adversarially robust generalization.
Nevertheless, a slightly non-linear classifier that utilizes thresholding turns out to recover the smaller sample complexity of standard generalization.
Since MNIST is a mostly binary dataset, our result provides evidence that $\ell_\infty$-robustness on MNIST is significantly easier than on other datasets.
Moreover, our results show that distributions with similar sample complexity for standard generalization can still exhibit considerably different sample complexity for robust generalization.

To complement our theoretical results, we conduct a range of experiments on MNIST, CIFAR10, and SVHN.
By subsampling the datasets at various rates, we study the impact of sample size on adversarial robustness.
When plotted as a function of training set size, our results show that the standard accuracy on SVHN indeed plateaus well before the adversarial accuracy reaches its maximum.
On MNIST, explicitly adding thresholding to the model during training significantly reduces the sample complexity, similar to our upper bound in the binary data model.
On CIFAR10, the situation is more nuanced because there are no known approaches that achieve more than 50\% accuracy even against a mild adversary.
But as we show in the next subsection, there is clear evidence for overfitting in the current state-of-the-art methods.

Overall, our results suggest that current approaches may be unable to attain higher adversarial accuracy on datasets such as CIFAR10 for a fundamental reason:
the dataset may not be large enough to train a standard convolutional network robustly.
Moreover, our lower bounds illustrate that the existence of adversarial examples should not necessarily be seen as a shortcoming of specific classification methods.
Already in a simple data model, adversarial examples \emph{provably} occur for any learning approach, even when the classifier already achieves high standard accuracy.
So while vulnerability to adversarial $\ell_\infty$-perturbations might seem counter-intuitive at first, in some regimes it is an unavoidable consequence of working in a statistical setting.

\begin{figure}[tb]
\begin{center}
\iftoggle{isicml}{
  \includegraphics[width=0.45\textwidth]{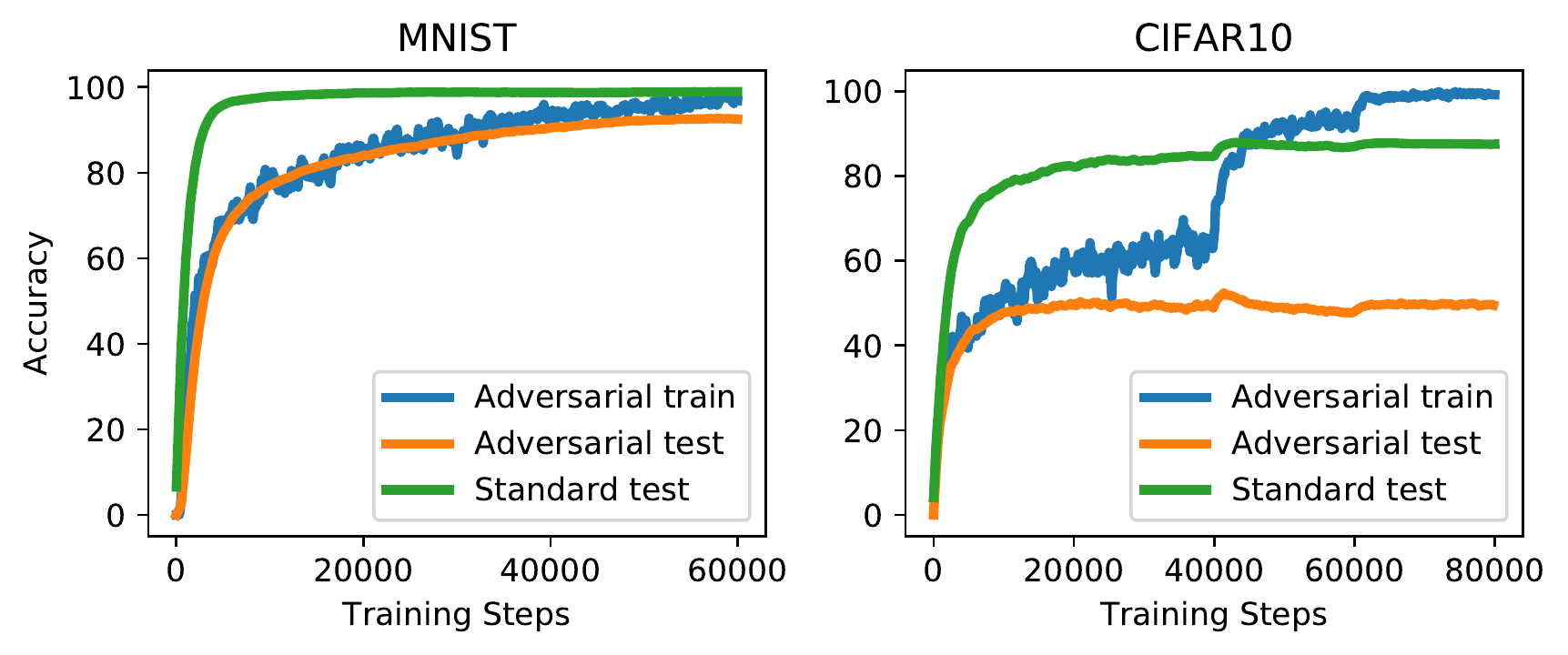}
}{
  \includegraphics[width=0.85\textwidth]{figures/overfit.pdf}

  \vspace{-.5cm}
}
\end{center}
\caption{Classification accuracies for robust optimization on MNIST and CIFAR10.
In both cases, we trained standard convolutional networks to be robust to $\ell_\infty$-perturbations of the input.
On MNIST, the robust test error closely tracks the corresponding training error and the model achieves high robust accuracy.
On CIFAR10, the model still achieves a good natural (non-adversarial) test error, but there is a significant generalization gap for the robust accuracy.
This phenomenon motivates our study of adversarially robust generalization.
}
\label{fig:intro}
\end{figure}

\subsection{A motivating example: Overfitting on CIFAR10}
Before we describe our main results, we briefly highlight the importance of generalization for adversarial robustness via two experiments on MNIST and CIFAR10.
In both cases, our goal is to learn a classifier that achieves good test accuracy even under $\ell_\infty$-bounded perturbations.
We follow the standard robust optimization approach \cite{wald1945statistical,ben2009robust,madry2017towards} -- also known as adversarial training \cite{SzegedyZSBEGF13,GoodfellowSS14} -- and (approximately) solve the saddle point problem
\[
  \min_\theta \mathop{\mathbb{E}}_{x} \left[ \max_{\| x' - x\|_\infty \leq \eps} \text{loss}(\theta, x') \right]
\]
via stochastic gradient descent over the model parameters $\theta$.
We utilize projected gradient descent for the inner maximization problem over allowed perturbations of magnitude $\eps$ (see \cite{madry2017towards} for details).
Figure \ref{fig:intro} displays the training curves for three quantities: (i) adversarial training error, (ii) adversarial test error, and (iii) standard test error.

The results show that on MNIST, robust optimization is able to learn a model with around 90\% adversarial accuracy and a relatively small gap between training and test error.
However, CIFAR10 offers a different picture.
Here, the model (a wide residual network \cite{ZK16}) is still able to fully fit the training set even against an adversary, but the generalization gap is significantly larger.
The model only achieves 47\% adversarial test accuracy, which is about 50\% lower than its training accuracy.\footnote{We remark that this accuracy is still currently the best published robust accuracy on CIFAR10 \cite{athalye2018obfuscated}.
For instance, contemporary approaches to architecture tuning do not yield better robust accuracies \cite{ZCLS18}.}
Moreover, the standard test error is about 87\%, so the failure of generalization indeed primarily occurs in the context of adversarial robustness.
This failure might be surprising particularly since properly tuned convolutional networks rarely overfit much on standard vision datasets.

\subsection{Outline of the paper}
In the next section, we describe our main theoretical results at a high level.
Sections~\ref{sec:linfty_gauss} and \ref{sec:bernlb_sketch} then provide more details for our lower bounds on $\ell_\infty$-robust generalization.
Section~\ref{sec:experiments} complements these results with experiments.
We conclude with a discussion of our results and future research directions.

\section{Theoretical Results}
\label{sec:contributions}
Our theoretical results concern statistical aspects of adversarially robust classification.
In order to understand how properties of data affect the number of samples needed for robust generalization, we study two concrete distributional models.
While our two data models are clearly much simpler than the image datasets currently being used in the experimental work on $\ell_\infty$-robustness, we believe that the simplicity of our models is a strength in this context.

After all, the fact that we can establish a separation between standard and robust generalization already in our Gaussian data model is evidence that the existence of adversarial examples for neural networks should not come as a surprise.
The same phenomenon (i.e., classifiers with just enough samples for high standard accuracy \emph{necessarily} being vulnerable to $\ell_\infty$- attacks) already occurs in much simpler settings such as a mixture of two Gaussians.

Also, our main contribution is a \emph{lower bound}.
So establishing a hardness result for a simple problem means that more complicated distributional setups that can ``simulate'' the Gaussian model directly inherit the same hardness.

Finally, as we describe in the subsection on the Bernoulli model, the benefits of the thresholding layer predicted by our theoretical analysis do indeed appear in experiments on MNIST as well.
  Since multiple defenses against adversarial examples have been primarily evaluated on MNIST \cite{kolter2017provable, raghunathan2018certified,sinha2017certifiable}, it is important to note that $\ell_\infty$-robustness on MNIST is a particularly easy case: adding a simple thresholding layer directly yields nearly state-of-the-art robustness against moderately strong adversaries ($\eps = 0.1$), without any further changes to the model architecture or training algorithm.

\subsection{The Gaussian model}
Our first data model is a mixture of two spherical Gaussians with one component per class.
\begin{definition}[Gaussian model]
  Let $\thetastar \in \R^d$ be the per-class mean vector and let $\sigma > 0$ be the variance parameter.
  Then the $(\thetastar, \sigma)$-Gaussian model is defined by the following distribution over $(x, y) \in \R^d \times \{\pm 1\}$:
  First, draw a label $y \in \{\pm 1\}$ uniformly at random.
  Then sample the data point $x \in \R^d$ from $\normal(y \cdot \thetastar, \sigma^2 \matI)$.
\end{definition}
While not explicitly specified in the definition, we will use the Gaussian model in the regime where the norm of the vector $\thetastar$ is approximately $\sqrt{d}$.
Hence the main free parameter for controlling the difficulty of the classification task is the variance $\sigma^2$, which controls the amount of overlap between the two classes.

To contrast the notions of ``standard'' and ``robust'' generalization, we briefly recap a standard definition of classification error.
\begin{definition}[Classification error]
  Let $\distP : \R^d \times \{\pm 1 \} \rightarrow \R$ be a distribution.
  Then the classification error $\beta$ of a classifier $f : \R^d \rightarrow \{ \pm 1\}$ is defined as $\beta \defeq \prob_{(x,y) \sim \distP}\brackets*{f(x) \ne y}$.
\end{definition}

Next, we define our main quantity of interest, which is an adversarially robust counterpart of the above classification error.
Instead of counting misclassifications under the data distribution, we allow a bounded worst-case perturbation before passing the perturbed sample to the classifier.
\begin{definition}[Robust classification error]
  \label{def:robusterror}
  Let $\distP : \R^d \times \{\pm 1 \} \rightarrow \R$ be a distribution and let $\perturbB : \R^d \rightarrow \powerset(\R^d)$ be a perturbation set.\footnote{We write $\powerset(\R^d)$ to denote the power set of $\R^d$, i.e., the set of subsets of $\R^d$.}
  Then the $\perturbB$-robust classification error $\beta$ of a classifier $f : \R^d \rightarrow \{ \pm 1\}$ is defined as $\beta \defeq \prob_{(x,y) \sim \distP}\brackets*{\, \exists \, x' \in \perturbB(x) \, : \: f(x') \ne y }$.
\end{definition}
Since $\ell_\infty$-perturbations have recently received a significant amount of attention, we focus on robustness to $\ell_\infty$-bounded adversaries in our work.
For this purpose, we define the perturbation set $\perturbB_\infty^\eps(x) = \{ x' \in \R^d \, | \, \norm{x' - x}_\infty \leq \eps\}$.
To simplify notation, we refer to robustness with respect to this set also as $\ell_\infty^\eps$-robustness.
As we remark in the discussion section, understanding generalization for other measures of robustness ($\ell_2$, rotatations, etc.) is an important direction for future work.

\paragraph{Standard generalization.}
The Gaussian model has one parameter for controlling the difficulty of learning a good classifier.
In order to simplify the following bounds, we study a regime where it is possible to achieve good \emph{standard} classification error from a single sample.\footnote{We remark that it is also possible to study a more general setting where standard generalization requires a larger number of samples.}
As we will see later, this also allows us to calibrate our two data models to have comparable standard sample complexity.

Concretely, we prove the following theorem, which is a direct consequence of Gaussian concentration.
Note that in this theorem we use a \emph{linear classifier}:
for a vector $w$, the linear classifier $f_w : \R^d \rightarrow \pmset$ is defined as $f_w(x) = \sign({\ip{w, x}})$.
\begin{theorem}
  \label{thm:gauss_standard_upper}
  Let $(\vx, y)$ be drawn from a $(\thetastar, \sigma)$-Gaussian model with $\norm{\thetastar}_2 = \sqrt{d}$ and $\sigma \, \leq \, c \cdot d^{\sfrac{1}{4}}$ where $c$ is a universal constant.
  Let $\what \in\R^d$ be the vector $\what = y \cdot \vx$.
  Then with high probability, the linear classifier $f_{\what}$ has classification error at most 1\%.
\end{theorem}
To minimize the number of parameters in our bounds, we have set the error probability to 1\%.
By tuning the model parameters appropriately, it is possible to achieve a vanishingly small error probability from a single sample (see Corollary \ref{cor:single_sample} in Appendix~\ref{app:gaussians_upper}).

\paragraph{Robust generalization.}
As we just demonstrated, we can easily achieve \emph{standard} generalization from only a single sample in our Gaussian model.
We now show that achieving a low $\ell_\infty$-\emph{robust} classification error requires significantly more samples.
To this end, we begin with a natural strengthening of Theorem \ref{thm:gauss_standard_upper} and prove that the (class-weighted) sample mean can also be a robust classifier (given sufficient data).
\begin{theorem}
  \label{thm:main_gauss_upper}
  Let $(\vx_1, y_1), \ldots, (\vx_n, y_n)$ be drawn i.i.d.\ from a $(\thetastar, \sigma)$-Gaussian model with $\norm{\thetastar}_2 = \sqrt{d}$ and $\sigma \leq c_1 d^{\sfrac{1}{4}}$.
  Let $\what \in\R^d$ be the weighted mean vector $\what= \frac{1}{n} \sum_{i=1}^{n} y_i \vx_i$.
  Then with high probability, the linear classifier $f_{\what}$ has $\ell_\infty^\eps$-robust classification error at most 1\% if
  \[
  n \; \geq \; \begin{cases} 1 \quad &\text{ for } \;\; \eps \, \leq \, \frac{1}{4}d^{-\sfrac{1}{4}} \\
  c_2 \, \eps^2\sqrt{d} & \text{ for } \; \; \frac{1}{4}d^{-\sfrac{1}{4}} \, \leq \, \eps \, \leq \, \frac{1}{4}\end{cases} \; .
  \]
\end{theorem}
We refer the reader to Corollary \ref{cor:gaussian_robust} in Appendix \ref{app:gaussians_upper} for the details.
As before, $c_1$ and $c_2$ are two universal constants.
Overall, the theorem shows that it is possible to learn an $\ell_\infty^\eps$-robust classifier in the Gaussian model as long as $\eps$ is bounded by a small constant and we have a large number of samples.

Next, we show that this significantly increased sample complexity is necessary.
Our main theorem establishes a lower bound for \emph{all} learning algorithms, which we formalize as functions from data samples to binary classifiers.
In particular, the lower bound applies not only to learning linear classifiers.
\begin{theorem}
  \label{thm:main_gauss_lower}
Let $g_n$ be any learning algorithm, i.e., a function from $n$ samples to a binary classifier $f_n$.
Moreover, let $\sigma = c_1 \cdot d^{\sfrac{1}{4}}$, let $\eps \geq 0$, and let $\theta \in \R^d$ be drawn from $\normal(0, \matI)$.
We also draw $n$ samples from the $(\theta, \sigma)$-Gaussian model.
Then the expected $\ell_\infty^\eps$-robust classification error of $f_n$ is at least $(1 - \sfrac{1}{d}) \frac{1}{2}$ if
\[
  n \; \leq \; c_2 \frac{\eps^2 \, \sqrt{d}}{\log d} \; .
\]
\end{theorem}
The proof of the theorem can be found in Corollary \ref{cor:gaussian_robust} (Appendix \ref{app:gaussian_lb}) and we provide a brief sketch in Section \ref{sec:linfty_gauss}.
It is worth noting that the classification error $\sfrac{1}{2}$ in the lower bound is tight.
A classifier that always outputs a fixed prediction trivially achieves perfect robustness on one of the two classes and hence robust accuracy $\sfrac{1}{2}$.

Comparing Theorems \ref{thm:main_gauss_upper} and \ref{thm:main_gauss_lower}, we see that the sample complexity $n$ required for robust generalization is bounded as
\[
  \frac{c}{\log d} \; \leq \; \frac{n}{\eps^2 \sqrt{d}} \; \leq \; c' \; .
\]
Hence the lower bound is nearly tight in our regime of interest.
When the perturbation has constant $\ell_\infty$-norm, the sample complexity of robust generalization is larger than that of standard generalization by $\sqrt{d}$, i.e., \emph{polynomial} in the problem dimension.
This shows that for high-dimensional problems, adversarial robustness can provably require a significantly larger number of samples.

Finally, we remark that our lower bound applies also to a more restricted adversary.
As we outline in Sections \ref{sec:linfty_gauss}, the proof uses only a single adversarial perturbation per class.
As a result, the lower bound provides \emph{transferable} adversarial examples and applies to worst-case distribution shifts without a classifier-adaptive adversary.
We refer the reader to Section \ref{sec:discussion} for a more detailed discussion.

\subsection{The Bernoulli model}
\label{sec:contrib_bern}
As mentioned in the introduction, simpler datasets such as MNIST have recently seen significant progress in terms of $\ell_\infty$-robustness.
We now investigate a possible mechanism underlying these advances.
To this end, we study a second distributional model that highlights how the data distribution can significantly affect the achievable robustness.
The second data model is defined on the hypercube $\{\pm 1\}^d$, and the two classes are represented by opposite vertices of that hypercube.
When sampling a datapoint for a given class, we flip each bit of the corresponding class vertex with a certain probability.
This data model is inspired by the MNIST dataset because MNIST images are close to binary (many pixels are almost fully black or white).
\begin{definition}[Bernoulli model]
  Let $\thetastar \in \{\pm 1\}^d$ be the per-class mean vector and let $\tau > 0$ be the class bias parameter.
  Then the $(\thetastar, \tau)$-Bernoulli model is defined by the following distribution over $(x, y) \in \{\pm 1\}^d \times \{\pm 1\}$:
  First, draw a label $y \in \{\pm 1\}$ uniformly at random from its domain.
  Then sample the data point $x \in \{\pm 1\}^d$ by sampling each coordinate $x_i$ from the distribution
  \[
    x_i \; = \; \begin{cases}
      \phantom{-}y \cdot \thetastar_i & \text{with probability} \; \sfrac{1}{2} + \tau \\
      -y \cdot \thetastar_i & \text{with probability} \; \sfrac{1}{2} - \tau \\
    \end{cases} \; .
  \]
\end{definition}
As in the previous subsection, the model has one parameter for controlling the difficulty of learning.
A small value of $\tau$ makes the samples less correlated with their respective class vectors and hence leads to a harder classification problem.
Note that both the Gaussian and the Bernoulli model are defined by simple sub-Gaussian distributions.
Nevertheless, we will see that they differ significantly in terms of robust sample complexity.

\paragraph{Standard generalization.}
As in the Gaussian model, we first calibrate the distribution so that we can learn a classifier with good \emph{standard} accuracy from a single sample.\footnote{To be precise, the two distributions have comparable sample complexity for standard generalization in the regime where $\sigma \approx \tau^{-1}$.}
The following theorem is a direct consequence of the fact that bounded random variables exhibit sub-Gaussian concentration.
\begin{theorem}
  \label{thm:bernoulli_standard_upper}
	Let $(\vx, y)$ be drawn from a $(\thetastar, \tau)$-Bernoulli model with $\tau \, \geq \, c\cdot d^{-\sfrac{1}{4}}$ where $c$ is a universal constant.
	Let $\what \in\R^d$ be the vector $\what = y \cdot \vx$.
	Then with high probability, the linear classifier $f_{\what}$ has classification error at most 1\%.
\end{theorem}
To simplify the bound, we have set the error probability to be 1\% as in the Gaussian model.
We refer the reader to Corollary \ref{cor:bernoulli_single_sample} in Appendix \ref{app:bernoulli_upper} for the proof.

\paragraph{Robust generalization.}
Next, we investigate the sample complexity of robust generalization in our Bernoulli model.
For {\em linear} classifiers, a small robust classification error again requires a large number of samples:
\begin{theorem}
\label{thm:main_bernoulli_lower}
Let $g_n$ be a linear classifier learning algorithm, i.e., a function from $n$ samples to a linear classifier $f_n$.
Suppose that we choose $\thetastar$ uniformly at random from $\{ \pm1\}^d$ and draw $n$ samples from the $(\thetastar, \tau)$-Bernoulli model with $\tau = c_1 \cdot d^{-\sfrac{1}{4}}$.
Moreover, let $\eps < 3 \tau$ and $0 < \gamma < \sfrac{1}{2}$.
Then the expected $\ell_\infty^\eps$-robust classification error of $f_n$ is at least $\frac{1}{2} - \gamma$ if
\[
  n \; \leq \; c_2 \frac{\eps^2\gamma^2 d}{\log \sfrac{d}{\gamma}} \; .
\]
\end{theorem}
We provide a proof sketch in Section \ref{sec:bernlb_sketch} and the full proof in Appendix \ref{app:bern_lb}.
At first, the lower bound for linear classifiers might suggest that $\ell_\infty$-robustness requires an inherently larger sample complexity here as well.
However, in contrast to the Gaussian model, non-linear classifiers can achieve a significantly improved robustness. 
In particular, consider the following thresholding operation $T : \R^d \rightarrow \R^d$ which is defined element-wise as
\[
  T(x)_i \; = \; \begin{cases} +1 & \quad \text{if } x_i \geq 0 \\
                               -1 & \quad \text{otherwise}
\end{cases} .
\]
It is easy to see that for $\eps < 1$, the thresholding operator undoes the action of any $\ell_\infty$-bounded adversary, i.e., we have $T(\perturbB_\infty^\eps(x)) = \{ x \}$ for any $x \in \{\pm 1\}^d$.
Hence we can combine the thresholding operator with the classifier learned from a single sample to get the following upper bound.
\begin{theorem}
  \label{thm:main_bernoulli_upper}
	Let $(\vx, y)$ be drawn from a $(\thetastar, \tau)$-Bernoulli model with $\tau \, \geq \, c\cdot d^{-\sfrac{1}{4}}$ where $c$ is a universal constant.
	Let $\what \in\R^d$ be the vector $\what = y \vx$.
	Then with high probability, the classifier $f_{\what} \circ T$ has $\ell_\infty^\eps$-robust classification error at most 1\% for any $\eps < 1$.
\end{theorem}
This theorem shows a stark contrast to the Gaussian case.
Although both models have similar sample complexity for standard generalization, there is a $\sqrt{d}$ gap between the $\ell_\infty$-robust sample complexity for the Bernoulli and Gaussian models.
This discrepancy provides evidence that robust generalization requires a more nuanced understanding of the data distribution than standard generalization.

In isolation, the thresholding step might seem specific to the Bernoulli model studied here.
However, our experiments in Section~\ref{sec:experiments} show that an explicit thresholding layer also significantly improves the sample complexity of training a robust neural network on MNIST.
We conjecture that the effectiveness of thresholding is behind many of the successful defenses against adversarial examples on MNIST (for instance, see Appendix C in \cite{madry2017towards}).

\section{Lower Bounds for the Gaussian Model}
\label{sec:linfty_gauss}
Recall our main theoretical result: In the Gaussian model, no algorithm can produce a robust classifier unless it has seen a large number of samples.
In particular, we give a nearly tight trade-off between the number of samples and the $\ell_{\infty}$-robustness of the classifier.
The following theorem is the technical core of this lower bound.
Combined with standard bounds on the $\ell_\infty$-norm of a random Gaussian vector, it gives Theorem \ref{thm:main_gauss_lower} from the previous section.
\begin{restatable}{theorem}{gaussianlbmain}
\label{thm:gauss_linf_lower}
Let $g_n$ be any learning algorithm, i.e., a function from $n$ samples in $\R^d \times \{\pm 1\}$ to a binary classifier $f_n$.
  Moreover, let $\sigma > 0$, let $\eps \geq 0$, and let $\theta \in \R^d$ be drawn from $\normal(0, \matI)$.
  We also draw $n$ samples from the $(\theta, \sigma)$-Gaussian model.
  Then the expected $\ell_\infty^\eps$-robust classification error of $f_n$ is at least
  \[
           \frac{1}{2} \probop_{v \sim \normal(0, \matI)} \brackets*{ \sqrt{\frac{n}{\sigma^2 + n}} \norm{v}_\infty \leq \eps   } \; .
  \]
\end{restatable}
Several remarks are in order. 
Since we lower bound the expected robust classification error for a distribution over the model parameters $\thetastar$, our result implies a lower bound on the minimax robust classification  error (i.e., minimum over learning algorithms, maximum over unknown parameters $\thetastar$).
Second, while we refer to the learning procedure as an algorithm, our lower bounds are information theoretic and hold irrespective of the computational power of this procedure.

Moreover, our proof shows that given the $n$ samples, there is a {\em single} adversarial perturbation that (a) applies to all learning algorithms, and (b) leads to at least a constant fraction of fresh samples being misclassified.
In other words, the same perturbation is transferable across examples as well as across architectures and learning procedures.
Hence our simple Gaussian data model already exhibits the transferability phenomenon, which has recently received significant attention in the deep learning literature (e.g., \cite{SzegedyZSBEGF13,TramerPGBM17,dezfooli2017universal}).

We defer a full proof of the theorem to Section \ref{app:gaussian_lb} of the supplementary material.
Here, we sketch the main ideas of the proof.

We fix an algorithm $g_n$ and let $S_n$ denote the set of $n$ samples given to the algorithm. We are interested in the expected robust classification error, which can be formalized as
\begin{align*}
  \E_{\theta^*} \E_{S_n}  \E_{y \sim \pm 1} \Pr_{x \sim \mathcal{N}(y\theta^*, \sigma^2 I)} \brackets*{ \exists x' \in \perturbB_{\infty}^\eps(x) : f_n(x') \neq y} \; .
\end{align*}
We swap the two outer expectations so the quantity of interest becomes
\begin{align*}
  \E_{S_n} \E_{\theta^*}  \E_{y \sim \pm 1} \Pr_{x \sim \mathcal{N}(y\theta^*, \sigma^2 I)} \brackets*{ \exists x' \in \perturbB_{\infty}^\eps(x) : f_n(x') \neq y } \; .
\end{align*}

Given the samples $S_n$, the posterior on $\theta^*$ is a Gaussian distribution with parameters defined by simple statistics of $S_n$ (the sample mean and the number of samples).
Since the new data point $x$ (to be classified) is itself drawn from a Gaussian distribution with mean $\theta^*$, the posterior distribution $\mu_{+}$ on the positive examples $x \sim \mathcal{N}(\theta^*, \sigma^2)$ is another Gaussian with a certain mean $\bar{z}$ and standard deviation $\sigma'$.
Similarly, the posterior distribution $\mu_{-}$ on the negative examples is a Gaussian with mean $-\bar{z}$ and the same  standard deviation $\sigma'$.
At a high level, we will now argue that the adversary can make the two posterior distributions $\mu_{-}$ and $\mu_{+}$ similar enough so that the problem becomes inherently noisy, preventing any classifier from achieving a high accuracy.

To this end, define the classification sets of $f_n$ as $A_{+} =  \{x \, | \, f_n(x) = +1\}$ and $A_{-} = \R^d \setminus A_{+}$.
This allows us to write the expected robust classification error as
\begin{align*}
  \E_{S_n} \E_{\theta^*} \parens*{\frac{1}{2} \Pr_{\mu_+}\brackets*{\perturbB_\infty^\eps(A_-)} + \frac{1}{2} \Pr_{\mu_-}\brackets*{\perturbB_\infty^\eps(A_+)}} \; .
\end{align*}

We now lower bound the inner probabilities by considering the fixed perturbation $\Delta=\vzbar$.
Note that a point $x \sim \mu_+$ is certainly misclassified if we have $\norm{\Delta}_\infty \leq \eps$ and $x - \Delta \in A_{-}$.
Thus the expected misclassification rate is at least $\mu_+(\{x \, | \,  x - \Delta \in A_{-}\}) = \mu_+(A_{-} + \Delta )$.\footnote{For a set $A$ and a vector $v$, we use the notation $A+v$ to denote the set $\{x + v : x \in A\}$.}
But since $\mu_+$ is simply a translated version of $N(0, {\sigma'}^2)$, this implies that
\begin{align*}
\Pr_{\mu_+}\brackets*{\perturbB_\infty^\eps(A_-)} \; \geq \; \mu_0(A_{-} + \Delta - \bar{z}) \; = \; \mu_0(A_-)
\end{align*}
where the distribution $\mu_0$ is the centered Gaussian $\mu_0 = \mathcal{N}(0, {\sigma'}^2)$.
Similarly,
\begin{align*}
\Pr_{\mu_-}\brackets*{\perturbB_\infty^\eps(A_+)} \; \geq \; \mu_0(A_{+} - \Delta + \bar{z}) \; = \; \mu_0(A_+).
\end{align*}

Since $\mu_0(A_-) + \mu_0(A_+) = 1$, this implies that the adversarial perturbation $-\bar{z}$ misclassifies in expectation half of the positively labeled examples, which completes the proof.
As mentioned above, the crucial step is that the posteriors $\mu_+$ and $\mu_-$ are similar enough so that we can shift both to the origin while still controlling the measure of the sets $A_-$ and $A_+$.

\section{Lower Bounds for the Bernoulli Model}
\label{sec:bernlb_sketch}
For the Bernoulli model, our lower bound applies only to \emph{linear} classifiers.
As pointed out in Section \ref{sec:contrib_bern}, non-linear classifiers do not suffer an increase in sample complexity in this data model.
We now give a high-level overview of our proof that the sample complexity for learning a linear classifier must increase as
\begin{equation}
  \label{eq:bernsketch}
  n \; \geq \; c \, \frac{\eps^2 d}{\log d} \; .
\end{equation}

At first, this lower bound may look stronger than in the Gaussian case, where Theorem \ref{thm:main_gauss_lower} established a lower bound of the form $\frac{\eps^2 \sqrt{d}}{\log d}$, i.e., with only a square root dependence on $d$.
However, it is important to note that the relevant $\ell_{\infty}$-robustness scale for linear classifiers in the Bernoulli model is on the order of $\Theta(\tau)$, whereas non-linear classifiers can achieve robustness for noise level $\eps$ up to $1$.
In particular, we prove that no linear classifier can achieve small $\ell_\infty$-robust classification error for $\eps > 3 \tau$ (see Lemma \ref{lem:linear_besteps} in Appendix \ref{app:bern_lb} for details).
Recall that we focus on the $\tau = \Theta(d^{-\frac 1 4})$ regime.
In this case, the lower bound in Equation \ref{eq:bernsketch} is on the order of $\sqrt{d}$ samples, which is comparable to the (nearly) tight bound for the Gaussian case.
This is no coincidence: for our noise parameters $\sigma \approx \tau^{-1} \approx d^{\frac 1 4}$, one can show that approximately $\sigma^2 = \sqrt{d}$ samples suffice to recover $\theta^*$ to sufficiently good accuracy.

The point of start of our proof of the lower bound for linear classifiers is the following observation. 
For an example $(x,y)$, a linear classifier with parameter vector $w$ robustly classifies the point $x$ if and only if
\[
  \inf_{\Delta : \norm{\Delta}_\infty \leq \eps} \ip{yw, x+\Delta} \; > \; 0 \; ,
\]
which is equivalent to
\[
  \ip{yw, x}  \; > \; \sup_{\Delta : \norm{\Delta}_\infty \leq \eps} \ip{yw, \Delta} \; .
\]
By the definition of dual norms, the supremum on the right hand size is thus equal to $\eps \norm{yw}_{1} = \eps \norm{w}_1$.  

The learning algorithm infers the parameter vector $w$ from a limited number of samples.
Since these samples are noisy copies of the unknown parameters $\thetastar$, the algorithm cannot be too certain of any single bit in $\theta^*$ (recall that  we draw $\theta^*$ uniformly from the hypercube).
We formalize this intuition in Lemma \ref{lem:bb_one_d} (Appendix \ref{app:bern_lb}) as a bound on the log odds given a sample $S$:
\[
  \log \frac{\Pr[\theta =+1 \mid S]}{\Pr[\theta =-1 \mid S]} \; .
\]
Given such a bound, we can analyze the uncertainty in the estimate $w$ by establishing an upper bound on the posterior $\abs*{\E[\thetastar_i | S]}$ for each $i \in [d]$.
This in turn allow us to bound $\E\brackets*{\ip{w, \thetastar} | S}$.
With control over this expectation, we can then relate the prediction $\ip{w, x}$ and the $\ell_1$-norm $\norm{w}_1$ via a tail bound argument.
We defer the details to Appendix \ref{app:bern_lb}.

\section{Experiments}
\label{sec:experiments}
We complement our theoretical results by performing experiments on multiple common datasets.

\subsection{Experimental setup}
We consider standard convolutional neural networks and train models on datasets of varying complexity.
Specifically, we study the MNIST~\cite{lecun1998mnist}, CIFAR-10~\cite{krizhevsky2009learning}, and SVHN~\cite{netzer2011reading} datasets.
The latter is particularly well-suited for our analysis since it contains a large number of training images (more than 600,000), allowing us to study adversarially robust generalization in the large dataset regime.

\paragraph{Model architecture.}
For MNIST, we use the simple convolution architecture obtained from the TensorFlow tutorial~\cite{TFtutorial}.
In order to prevent the model from overfitting when trained on small data samples, we regularize the model by adding weight decay with parameter $0.5$ to the training loss.
For CIFAR-10, we consider a standard ResNet model~\cite{ResnetPaper}.
It has 4 groups of residual layers with filter sizes (16, 16, 32, 64) and 5 residual units each.
On SVHN, we also trained a network of larger capacity (filter sizes of $(16, 64, 128, 256)$ instead of $(16, 16, 32, 64)$) in order to perform well on the harder problems with larger adversarial perturbations.
All of our models achieve close to state-of-the-art performance on the respective benchmark.

\paragraph{Robust optimization.}
We perform robust optimization to train our classifiers.
In particular, we train against a projected gradient descent (PGD) adversary, starting from a random initial perturbation of the training datapoint (see \cite{madry2017towards} for more details).
We consider adversarial perturbations in  $\ell_\infty$ norm, performing PGD updates of the form
\[
  x_{t+1} = \Pi_{\perturbB^\eps_\infty(x_0)}\left(
                x_t + \lambda \cdot \sign(\nabla\loss (x_t))
            \right)
\]
for some step size $\lambda$.
Here, $\loss$ denotes the loss of the model, while $\Pi_{\perturbB^r_\infty(x)}(z)$ corresponds to projecting $z$ onto the $\ell_\infty$ ball of radius $r$ around $x$.
On MNIST, we perform $20$ steps of PGD, while on CIFAR-10 and SVHN we perform $10$ steps.
We evaluate all networks against a $20$-step PGD adversary.
We choose the PGD step size to be $2.5\cdot \eps/k$, where $\eps$ denotes the maximal allowed perturbation and $k$ is the total number of steps.
This allows PGD to reach the boundary of the optimization region within $\frac{k}{2.5}$ steps from any starting point.

\subsection{Empirical sample complexity evaluation}
\label{sec:sample_complexity}
We study how the generalization performance of adversarially robust networks varies with the size of the training dataset.
To do so, we train networks with a specific $\ell_\infty$ adversary (for some fixed $\eps_{train}$) while reducing the size of the training set.
The training subsets are produced by randomly sub-sampling the complete dataset in a class-balanced fashion.
When increasing the number of samples, we ensure that each dataset is a superset of the previous one.

We then evaluate the robustness of each trained network to perturbations of varying magnitude ($\eps_{test}$).
For each choice of training set size $N$ and fixed attack $\eps_{test}$, we select the best performance achieved across all hyperparameters settings (training perturbations $\eps_{train}$ and model size).
On all three datasets, we observed that the best natural accuracy is usually achieved for the naturally trained network, while the best adversarial accuracy for almost all values of $\eps_{test}$ was achieved when training with the largest $\eps_{train}$.
We maximize over the hyperparameter settings since we are not interested in the performance of a specific model, but rather in the inherent generalization properties of the dataset independently of the classifier used.
The results of these experiments are shown in Figure~\ref{fig:sample_complexity} for each dataset.

The plots clearly demonstrate the need for more data to achieve adversarially robust generalization.
For any fixed test set accuracy, the number of samples needed is significantly higher for robust generalization.
In the SVHN experiments (where we have sufficient training samples to observe plateauing behavior), the natural accuracy reaches its maximum with significantly fewer samples than the adversarial accuracy.
We report more details of our experiments in Section~\ref{sec:additional_experiments} of the supplementary material.

\begin{figure*}[!t]
	\vskip 0.2in
	\centering
	\includegraphics[width=1.0\linewidth] {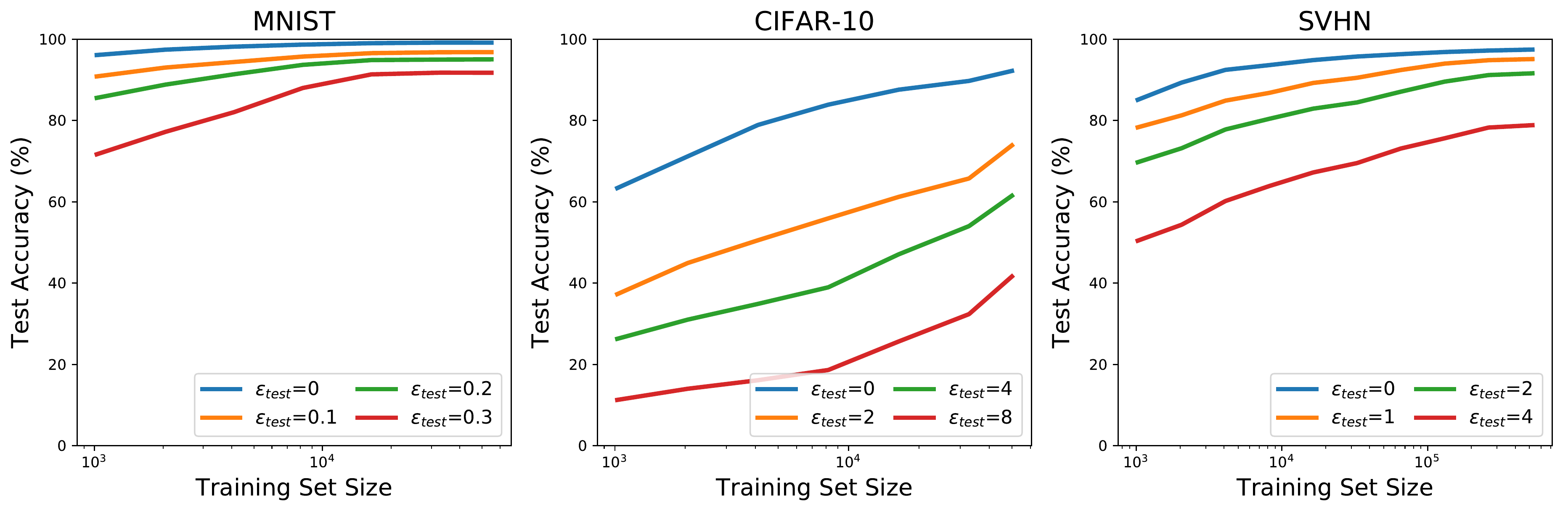}
	\caption{Adversarially robust generalization performance as a function of training data size for $\ell_\infty$ adversaries on the MNIST, CIFAR-10 and SVHN datasets.
    For each choice of training set size and $\eps_{test}$, we plot the best performance achieved over $\eps_{train}$ and network capacity.
  This clearly shows that achieving a certain level of adversarially robust generalization requires significantly more samples than achieving the same level of standard generalization.}
	\label{fig:sample_complexity}
	\vskip -0.2in
\end{figure*} 

\subsection{Thresholding experiments}
Motivated by our theoretical study of the Bernoulli model, we investigate whether thresholding can also improve the sample complexity of robust generalization against an $\ell_\infty$ adversary on a real dataset.
MNIST is a natural candidate here since the images are nearly black-and-white and hence lie close to vertices of a hypercube (as in the Bernoulli model).
This is further motivated by experimental evidence indicating that adversarially robust networks on MNIST learn such thresholding filters when trained adversarially \cite{madry2017towards}.

We repeat the sample complexity experiments performed in Section~\ref{sec:sample_complexity} with networks where thresholding filters are explicitly encoded in the model.
Here, we replace the first convolutional layer with a fixed thresholding layer consisting of two channels, $ReLU(x - \eps_{train})$ and $ReLU(x - (1 - \eps_{train}))$, where $x$ is the input image.
Results from networks trained with this thresholding layer are shown in Figure~\ref{fig:thresholding}.
For naturally trained networks, we use a value of $\eps=0.1$ for the thresholding filters, whereas for adversarially trained networks we set $\eps = \eps_{train}$.
For each data subset size and test perturbation $\eps_{test}$, we plot the best test accuracy achieved over networks trained with different thresholding filters, i.e., different values of $\eps$.
We separately show the effect of explicit thresholding in such networks when they are trained naturally or adversarially using PGD.
As predicted by our theory, the networks achieve good adversarially robust generalization with significantly fewer samples when thresholding filters are added. Further, note that adding a simple thresholding layer directly yields nearly state-of-the-art robustness against moderately strong adversaries ($\eps = 0.1$), without any other modifications to the model architecture or training algorithm.
It is also worth noting that the thresholding filters could have been learned by the original network architecture, and that this modification only decreases the capacity of the model.
Our findings emphasize network architecture as a crucial factor for learning adversarially robust networks from a limited number of samples.

We also experimented with thresholding filters on the CIFAR10 dataset, but did not observe any significant difference from the standard architecture.
This agrees with our theoretical understanding that thresholding helps primarily in the case of (approximately) binary datasets.

\begin{figure*}[!t]
	\vskip 0.2in
	\begin{center}
		\centerline{\includegraphics[width=1.0\linewidth]{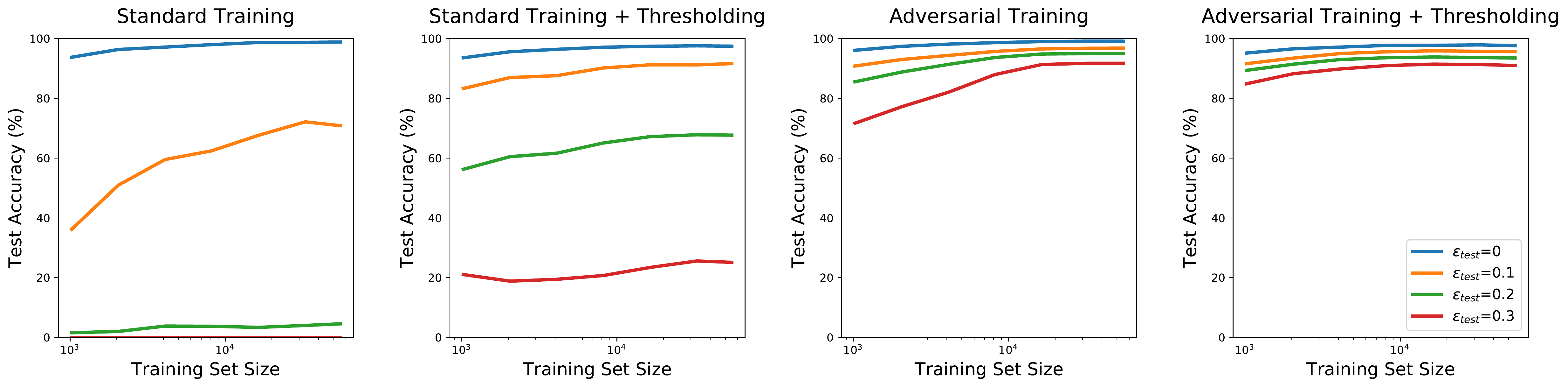}}
		\caption{Adversarial robustness to $\ell_\infty$ attacks on the MNIST dataset for a simple convolution network \cite{madry2017towards} with and without explicit thresholding filters.
        For each training set size choice and $\eps_{test}$, we report the best test set accuracy achieved over choice of thresholding filters and $\eps_{train}$.
        We observe that introducing thresholding filters significantly reduces the number of samples needed to achieve good adversarial generalization.}
		\label{fig:thresholding}
	\end{center}
	\vskip -0.2in
\end{figure*} 

\section{Related Work}
Due to the large body of work on adversarial robustness, we focus on related papers that also provide theoretical explanations for adversarial examples.
Compared to prior work, the main difference of our approach is the focus on generalization.
Most related papers study robustness either without the learning context, or in the limit as the number of samples approaches infinity.
As a result, finite sample phenomena do not arise in these theoretical approaches.
As we have seen in Figure \ref{fig:intro}, adversarial examples are currently a failure of generalization from a limited training set.
Hence we believe that studying robust generalization is an insightful avenue for understanding adversarial examples.

\begin{itemize}[leftmargin=0.5cm]
\item \citet{wang2017adversarial} study the adversarial robustness of nearest neighbor classifiers.
In contrast to our work, the authors give theoretical guarantees for a specific classification algorithm.
We focus on the inherent sample complexity of adversarially robust generalization independently of the learning method.
Moreover, our results hold for finite sample sizes while the results in \cite{wang2017adversarial} are only asymptotic.

\item Recent work by \citet{gilmer2017spheres} explores a specific distribution where robust learning is empirically difficult with overparametrized neural networks.\footnote{It is worth noting that the distribution in \cite{gilmer2017spheres} has only one degree of freedom. Hence we conjecture that the observed difficulty of robust learning in their setup is due to the chosen model class and not due to an information-theoretic limit as in our work.}
The main phenomenon is that even a small natural error rate on their dataset translates to a large adversarial error rate.
Our results give a more nuanced picture that involves the sample complexity required for generalization.
In our data models, it is possible to achieve an error rate that is essentially zero by using a very small number of samples, whereas the adversarial error rate is still large unless we have seen a lot of samples.

\item \citet{FMF16} relate the robustness of linear and non-linear classifiers to adversarial and (semi-)random perturbations.
Their work studies the setting where the classifier is fixed and does not encompass the learning task.
We focus on generalization aspects of adversarial robustness and provide upper and lower bounds on the sample complexity.
Overall, we argue that adversarial examples are inherent to the statistical setup and not necessarily a consequence of a concrete classifier model.

\item The work of \citet{XCM2009} establishes a connection between robust optimization and regularization for linear classification.
In particular, they show that robustness to a specific perturbation set is exactly equivalent to the standard support vector machine.
The authors give asymptotic consistency results under a robustness condition, but do not provide any finite sample guarantees.
In contrast, our work considers specific distributional models where we can demonstrate a clear gap between robust and standard generalization.

\item \citet{PMSW18} discuss adversarial robustness at the population level.
They assume the existence of an adversary that can significantly increase the loss for \emph{any} hypothesis in the hypothesis class.
By definition, robustness against adversarial perturbations is impossible in this regime.
As demonstrated in Figure \ref{fig:intro}, we instead conjecture that current classification models are not robust to adversarial examples because they fail to generalize.
Hence our results concern generalization from a finite number of samples.
We show that even when the hypothesis class is large enough to achieve good robust classification error, the sample complexity of robust generalization can still be significantly bigger than that of standard generalization.

\item In a recent paper, \citet{FFF18} also give provable lower bounds for adversarial robustness.
There are several important differences between their work and ours.
At a high level, the results in \cite{FFF18} state that there are fundamental limits for adversarial robustness that apply to \emph{any} classifier.
As pointed out by the authors, their bounds also apply to the human visual system.
However, an important aspect of adversarial examples is that they often fool current classifiers, yet are still easy to recognize for humans.
Hence we believe that the approach in \cite{FFF18} does not capture the underlying phenomenon since it does not distinguish between the robustness of current artificial classifiers and the human visual system.

Moreover, the lower bounds in \cite{FFF18} do not involve the training data and consequently apply in the limit where an infinite number of samples is available.
In contrast, our work investigates how the amount of available training data affects adversarial robustness.
As we have seen in Figure \ref{fig:intro}, adversarial robustness is currently an issue of \emph{generalization}.
In particular, we can train classifiers that achieve a high level of robustness on the CIFAR10 training set, but this robustness does not transfer to the test set.
Therefore, our perspective based on adversarially robust generalization more accurately reflects the current challenges in training robust classifiers.

Finally, \citet{FFF18} utilize the notion of a latent space for the data distribution in order to establish lower bounds that apply to any classifier.
While the existence of generative models such as GANs provides empirical evidence for this assumption, we note that it does not suffice to accurately describe the robustness phenomenon.
For instance, there are multiple generative models that produce high-quality samples for the MNIST dataset, yet there are now also several successful defenses against adversarial examples on MNIST.
As we have shown in our work, the fine-grained properties of the data distribution can have significant impact on how hard it is to learn a robust classifier.

\end{itemize}

\paragraph{Margin-based theory.} There is a long line of work in machine learning on exploring the connection between various notions of margin and generalization, e.g., see \citep{shaishalev2014} and references therein.
In this setting, the $\ell_p$ margin, i.e., how robustly classifiable the data is for $\ell_p^*$-bounded classifiers, enables dimension-independent control of the sample complexity.
However, the sample complexity in concrete distributional models can often be significantly smaller than what the margin implies.
As we will see next, standard margin-based bounds do not suffice to demonstrate a gap between robust and benign generalization for the distributional models studied in our work.

First, we briefly remind the reader about standard margin-based results (see Theorem 15.4 in \citep{shaishalev2014} for details).
For a dataset that has bounded $\ell_2$ norm $\rho$ and $\ell_2$ margin $\gamma$, the classification error of the hard-margin SVM scales as
\[
  \sqrt{\frac{\parens*{\sfrac{\rho}{\gamma}}^2}{n}}
\]
where $n$ is the number of samples.
To illustrate this bound, consider the Gaussian model in the regime $\sigma = \Theta(d^{\sfrac{1}{4}})$ where a single sample suffices to learn a classifier with low error (see Theorem~\ref{thm:gauss_standard_upper}).
The standard bound on the norm of an i.i.d.\ Gaussian vector shows that we have a data norm bound $\rho = \Theta(d^{\sfrac{3}{4}})$ with high probability.
While the Gaussian model is not strictly separable in any regime, we can still consider the probability that a sample achieves at least a certain margin:
\[
  \probop_{z \sim \normal(0, \sigma^2 \matI)} \brackets*{ \frac{\ip{z, \thetastar}}{\norm{\thetastar}_2}  \, \geq \rho } \; \geq \; 1 - \delta \; .
\]
A simple calculation shows that for $\norm{\thetastar}_2 = \sqrt{d}$ (as in our earlier bounds), the Gaussian model does not achieve margin $\gamma \geq \sqrt{d}$ even at the quantile $\delta = \sfrac{1}{2}$.
Hence the margin-based bound would indicate a sample complexity of $\Omega(d^{\sfrac{1}{4}})$ already for \emph{standard} generalization, which obscures the dichotomy between standard and robust sample complexity.

\paragraph{Robust statistics.} An orthogonal line of work in robust statistics studies robustness of estimators to corruption of \emph{training} data~\citep{Huber81a}.
This notion of robustness, while also important, is not directly relevant to the questions addressed in our work.

\section{Discussion and Future Directions}
\label{sec:discussion}
The vulnerability of neural networks to adversarial perturbations has recently been a source of much discussion and is still poorly understood.
Different works have argued that this vulnerability stems from their discontinuous nature~\citep{SzegedyZSBEGF13}, their linear nature~\citep{GoodfellowSS14}, or is a result of high-dimensional geometry and independent of the model class~\cite{gilmer2017spheres}. 
Our work gives a more nuanced picture.
We show that for a natural data distribution (the Gaussian model), the model class we train does not matter and a standard linear classifier achieves optimal robustness.
However, robustness also strongly depends on properties of the underlying data distribution.
For other data models (such as MNIST or the Bernoulli model), our results demonstrate that non-linearities are indispensable to learn from few samples.
This dichotomy provides evidence that defenses against adversarial examples need to be tailored to the specific dataset and hence may be more complicated than a single, broad approach.
Understanding the interactions between robustness, classifier model, and data distribution from the perspective of generalization is an important direction for future work.

What do our results mean for robust classification of real images?
Our Gaussian lower bound implies that if an algorithm works for all (or most) settings of the unknown parameter $\thetastar$, then achieving strong $\ell_\infty$-robustness requires a sample complexity increase that is polynomial in the dimension.
There are a few different ways this lower bound could be bypassed.
After all, it is conceivable that the noise scale $\sigma$ is significantly smaller for real image datasets, making robust classification easier.
And even if that was not the case, a good algorithm could work for the parameters $\thetastar$ that correspond to real datasets while not working for most other parameters.
To accomplish this, the algorithm would implicitly or explicitly have prior information about the correct $\thetastar$.
While some prior information is already incorporated in the model architectures (e.g., convolutional and pooling layers), the conventional wisdom is often not to bias the neural network with our priors.
Our work suggests that there are trade-offs with robustness here and that adding more prior information could help to learn more robust classifiers.

The focus of our paper is on adversarial perturbations in a setting where the test distribution (before the adversary's action) is the same as the training distribution.
While this is a natural scenario from a security point of view, other setups can be more relevant in different robustness contexts.
For instance, we may want a classifier that is robust to small changes between the training and test distribution.
This can be formalized as the classification accuracy on \emph{unperturbed} examples coming from an \emph{adversarially} modified distribution.
Here, the power of the adversary is limited by how much the test distribution can be modified, and the adversary is not allowed to perturb individual samples coming from the modified test distribution.
Interestingly, our lower bound for the Gaussian model also applies to such worst-case distributional shifts.
In particular, if the adversary is allowed to shift the mean $\thetastar$ by a vector in $\perturbB_\infty^\eps$, our proof sketched in Section \ref{sec:linfty_gauss} transfers to the distribution shift setting.
Since the lower bound relies only on a single universal perturbation, this perturbation can also be applied directly to the mean vector.

\paragraph{Future directions.} 
Several questions remain.
We now provide a list of concrete directions for future work on robust generalization.

\begin{description}
\item[Stronger lower bounds.]
  An interesting aspect of adversarial examples is that the adversary can often fool the classifier on most inputs \cite{SzegedyZSBEGF13,CarliniW16a}.
While our results show a lower bound for classification error $\sfrac{1}{2}$, it is conceivable that misclassification rates much closer to 1 are unavoidable for at least one of the two classes (or equivalently, when the adversary is allowed to pick the class label).
In order to avoid degenerate cases such as achieving robustness by being the constant classifier, it would be interesting to study regimes where the classifier has high standard accuracy but does not achieve robustness yet.
In such a regime, does good standard accuracy imply that the classifier is vulnerable to adversarial perturbations on almost all inputs?

\item[Different perturbation sets.]
Depending on the problem setting, different perturbation sets are relevant.
Due to the large amount of empirical work on $\ell_\infty$ robustness, our paper has focused on such perturbations.
From a security point of view, we want to defend against perturbations that are imperceptible to humans.
While this is not a well-defined concept, the class of small $\ell_\infty$-norm perturbations should be contained in any reasonable definition of imperceptible perturbations.
However, changes in different $\ell_p$ norms \cite{SzegedyZSBEGF13,MoosDez16,CarliniW16a}, sparse perturbations \cite{PapernotMJFCS16,CarliniW16,NK17,SVS17}, or mild spatial transformations can also be imperceptible to a human \cite{CJBWMD18}.
In less adversarial settings, more constrained and lower-dimensional perturbations such as small rotations and translations may be more appropriate \cite{ETTSM17}.
Overall, understanding the sample complexity implications of different perturbation sets is an important direction for future work.

\item[Further notions of test time robustness.]
  As mentioned above, less adversarial forms of robustness may be better suited to model challenges arising outside security. 
How much easier is it to learn a robust classifier in more benign settings?
This question is naturally related to problems such as transfer learning and domain adaptation.

\item[Broader classes of distributions.]
Our results directly apply to two concrete distributional models.
While the results already show interesting phenomena and are predictive of behavior on real data, understanding the robustness properties for a broader class of distributions is an important direction for future work.
Moreover, it would be useful to understand what general properties of distributions make robust generalization hard or easy.

\item[Wider sample complexity separations.]
In our work, we show a separation of $\sqrt{d}$ between the standard and robust sample complexity for the Gaussian model.
It is open whether larger gaps are possible.
Note that for large adversarial perturbations, the data may no longer be robustly separable which leads to trivial gaps in sample complexity, simply because the harder robust generalization problem is impossible to solve.
Hence this question is mainly interesting in the regime where a robust classifier exists in the model class of interest.

\item[Robustness in the PAC model.]
Our focus has been on robust learning for specific distributions without any limitations on the hypothesis class.
A natural dual perspective is to investigate robust learning for specific hypothesis classes, as in the probably approximately correct (PAC) framework.
For instance, it is well known that the sample complexity of learning a half space in $d$ dimensions is $O(d)$.
Does this sample complexity also suffice to learn in the presence of an adversary at test time?
While robustness to adversarial training noise has been studied in the PAC setting (e.g., see \cite{KL93,KSS94,BEK99}), we are not aware of similar work on test time robustness.
\end{description}

\section*{Acknowledgements}
Ludwig Schmidt is supported by a Google PhD Fellowship.
During this research project, Ludwig was also a research fellow at the Simons Institute for the Theory of Computing, an intern in the Google Brain team, and a visitor at UC Berkeley.
Shibani Santurkar is supported by the National Science Foundation (NSF) under grants IIS-1447786, IIS-1607189, and CCF-1563880, and the Intel Corporation.
Dimitris Tsipras was supported in part by the NSF grant CCF-1553428.
Aleksander M\k{a}dry was supported in part by an Alfred P.~Sloan Research Fellowship, a Google Research Award, and the NSF grant CCF-1553428.

\printbibliography

\appendix

\section{Omitted proofs for the Gaussian model}
\label{app:gaussians}
\subsection{Upper bounds}
\label{app:gaussians_upper}
We begin with standard results about (sub)-Gaussian concentration in Fact \ref{fact:gaussian_norm} and Lemmas \ref{lem:gaussian_norm} to \ref{lem:unit_ip}.
These results show that a class-weighted average of sufficiently many samples from the Gaussian model achieves a large inner product with the unknown mean vector.
Lemma \ref{lem:gaussian_classification} then relates the inner product between a linear classifier and the mean vector to the classification accuracy.
Theorem \ref{thm:gaussian_standard} uses the lemmas to establish our main theorem for standard generalization.
Corollary \ref{cor:single_sample} instantiates the bound for learning from one sample.
After further simplification, this yields Theorem \ref{thm:gauss_standard_upper} from the main text.

For robust generalization, we first relate the inner product between a linear classifier and the unknown mean vector to the robust classification accuracy in Lemma \ref{lem:robustupper}.
Similar to the standard classification error, Theorem \ref{thm:gausslinf} and Corollary \ref{cor:gausslinf} then yield our upper bounds for robust generalization.
Simplifying Corollary \ref{cor:gausslinf} further gives Theorem \ref{thm:main_gauss_upper} from the main text.
\begin{fact}
  \label{fact:gaussian_norm}
  Let $\vz \in \R^d$ be drawn from a centered spherical Gaussian, i.e., $\vz \sim \normal_d(0, \sigma^2 \matI)$ where $\sigma > 0$.
  Then we have $\prob[ \norm{\vz}_2 \geq \sigma \sqrt{d} + t ] \leq e^{-t^2 / (2 \sigma^2)}$ \; .
\end{fact}
\begin{proof}
  We refer the reader to Example 5.7 in \cite{BLM2013} for a reference of this standard result.
Combined with $\E[ \norm{\vz}_2] \leq \sigma \sqrt{d}$, which is obtained from Jensen's Inequality, the aforementioned example gives the desired upper tail bound.
\end{proof}

\begin{lemma}
  \label{lem:gaussian_norm}
  Let $\vz_1, \ldots, \vz_n \in \R^d$ be drawn i.i.d.\ from a spherical Gaussian, i.e., $\vz_i \sim \normal_d(\vmu, \sigma^2 \matI)$ where $\vmu \in \R^d$ and $\sigma > 0$.
  Let $\vzbar \in \R^d$ be the sample mean vector $\vzbar = \frac{1}{n} \sum_{i=1}^{n} \vz_i$.
  Finally, let $\delta > 0$ be the target probability.
  Then we have
  \[
    \prob \brackets*{ \norm{\vzbar}_2 \geq \norm{\vmu}_2 + \frac{\sigma \parens*{\sqrt{d} + \sqrt{2 \log\sfrac{1}{\delta}}}}{\sqrt n} } \, \leq \, \delta \; .
  \]
\end{lemma}
\begin{proof}
  Since each $\vz_i$ has the same distribution as $\vmu + \vg_i$ for $\vg_i \sim \normal_d(0, \sigma^2 \matI)$, we can bound the desired tail probability for
  \begin{align*}
    \vzbar \; &= \; \frac{1}{n} \sum_{i=1}^n \vmu + \vg_i \\
              &= \; \vmu + \frac{1}{n} \sum_{i=1}^n \vg_i  \; .
  \end{align*}
  Morever, the average of the $\vg_i$ has the same distribution as $\vgbar \sim \normal_d(0, \frac{\sigma^2}{n} \matI)$.
  Hence it suffices to bound the tail of $\norm{\vmu + \vgbar}_2$.
  For any $c \geq 0$, applying the triangle inequality then gives
  \begin{align*}
    \prob[\norm{\vzbar}_2 \geq \norm{\vmu}_2 + c] \; &= \; \prob[\norm{\vmu + \vgbar}_2 \geq \norm{\vmu}_2 + c] \\ 
      & \leq \; \prob[ \norm{\vgbar}_2 \geq c] \; .
  \end{align*}
  Setting $c = \sigma \sqrt{\sfrac{d}{n}} + t$ with
  \[
    t \; = \; \sigma \sqrt{\frac{2 \log \sfrac{1}{\delta}}{n}}
  \]
  and substituting into Fact \ref{fact:gaussian_norm} then gives the desired result.
\end{proof}

For convenient use in our later theorems, we instantiate Lemma \ref{lem:gaussian_norm} with the parameters most relevant for our Gaussian model.
In particular, the norm of the mean vector $\vmu$ is $\sqrt{d}$ and we are interested in up to exponentially small failure probability $\delta$ (but not necessarily smaller).
\begin{lemma}
  \label{lem:gaussian_norm2}
  Let $\vz_1, \ldots, \vz_n \in \R^d$ be drawn i.i.d.\ from a spherical Gaussian with mean norm $\sqrt{d}$, i.e., $\vz_i \sim \normal_d(\vmu, \sigma^2 \matI)$ where $\vmu \in \R^d$, $\norm{\vmu}_2 = \sqrt{d}$, and $\sigma > 0$.
  Let $\vzbar \in \R^d$ be the sample mean vector $\vzbar = \frac{1}{n} \sum_{i=1}^{n} \vz_i$.
  Then we have
  \[
    \prob\brackets*{ \norm{\vzbar}_2 \, \geq \, \parens*{1 + \frac{2 \sigma}{\sqrt{n}}} \sqrt{d} } \; \leq \; e^{-\sfrac{d}{2}} \; .
  \]
\end{lemma}
\begin{proof}
  We substitute into Lemma \ref{lem:gaussian_norm} with $\norm{\vmu}_2 = \sqrt{d}$ and $\sqrt{2 \log \sfrac{1}{\delta}} = \sqrt{d}$.
\end{proof}

\begin{lemma}
  \label{lem:gaussian_ip1}
  Let $\vz_1, \ldots, \vz_n \in \R^d$ be drawn i.i.d.\ from a spherical Gaussian, i.e., $\vz_i \sim \normal_d(\vmu, \sigma^2 \matI)$ where $\vmu \in \R^d$ and $\sigma > 0$.
  Let $\vzbar \in \R^d$ be the mean vector $\vzbar = \frac{1}{n} \sum_{i=1}^{n} \vz_i$.
  Finally, let $\delta > 0$ be the target probability.
  Then we have
  \[
  \prob \brackets*{ \ip{\vzbar, \vmu} \, \leq \, \norm{\vmu}_2^2 - \sigma \norm{\mu}_2 \sqrt{\frac{2 \log \sfrac{1}{\delta}}{n}} \, } \, \leq \, \delta \; .
  \]
\end{lemma}
\begin{proof}
  As in Lemma \ref{lem:gaussian_norm}, we use the fact that $\vzbar$ has the same distribution as $\vmu + \vgbar$ where $\vgbar \sim \normal_d(0, \frac{\sigma^2}{n} \matI)$.
  For any $t \geq 0$, this allows us to simplify the tail event to 
  \[
    \prob \brackets*{ \ip{\vzbar, \vmu} \leq \norm{\vmu}_2^2 - t } \; = \; \prob \brackets*{ \ip{\vgbar, \vmu} \leq -t} \; .
  \]
  The right hand side can now be simplified to $\prob\brackets{h \geq t}$ where $h \sim \normal(0, \sfrac{\sigma^2 \norm{\vmu}_2^2}{n})$.
  Invoking the standard sub-Gaussian tail bound
  \[
    \prob\brackets*{h \geq t} \; \leq \; \exp\parens*{-\frac{n \cdot t^2}{2 \sigma^2 \norm{\mu}_2^2}}
  \]
  and substituting $t \; = \; \sigma \norm{\mu}_2 \sqrt{\frac{2 \log\sfrac{1}{\delta}}{n}}$ then gives the desired result.
\end{proof}

\begin{lemma}
  \label{lem:unit_ip}
  Let $\vz_1, \ldots, \vz_n \in \R^d$ be drawn i.i.d.\ from a spherical Gaussian with mean norm $\sqrt{d}$, i.e., $\vz_i \sim \normal_d(\vmu, \sigma^2 \matI)$ where $\vmu \in \R^d$, $\norm{\vmu}_2 = \sqrt{d}$, and $\sigma > 0$.
  Let $\vzbar \in \R^d$ be the sample mean vector $\vzbar = \frac{1}{n} \sum_{i=1}^{n} \vz_i$ and let $\what \in \R^d$ be the unit vector in the direction of $\vzbar$, i.e., $\what = \sfrac{\vzbar}{\norm{\vzbar}_2}$.
  Then we have
  \[
    \prob\brackets*{ \ip{\what, \vmu} \, \leq \, \frac{2 \sqrt{n} - 1}{2\sqrt{n} + 4 \sigma} \sqrt{d} } \; \leq \; 2 \exp\parens*{-\frac{d}{8(\sigma^2 + 1)}} \; .
  \]
\end{lemma}
\begin{proof}
  We invoke Lemma \ref{lem:gaussian_norm2}, which yields
  \[
    \norm{\vzbar}_2 \, \leq \, \parens*{1 + \frac{2 \sigma}{\sqrt{n}}} \sqrt{d}
  \]
  with probability $1 - e^{\sfrac{-d}{2}}$.
  Moreover, we invoke Lemma \ref{lem:gaussian_ip1} with $\delta = e^{\sfrac{-d}{8\sigma^2}}$ and $\norm{\mu}_2 = \sqrt{d}$ to get
  \[
    \ip{\vzbar, \vmu} \, \geq \, d - \frac{d}{2\sqrt{n}}
  \]
  with probability $1 - e^{\sfrac{-d}{8\sigma^2}}$.
  We continue under both events, which yields the desired overall failure probability $2 e^{\sfrac{-d}{2}}$.

  We now have
  \begin{align*}
    \ip{\what, \vmu} \; &= \; \frac{\ip{\vzbar, \vmu}}{ \norm{\vzbar}_2} \\
                        &\geq \; \frac{\parens*{1 - \frac{1}{2\sqrt{n}}} d}{\norm{\vzbar}_2} \\
                        &\geq \; \frac{\parens*{1 - \frac{1}{2\sqrt{n}}} d}{(1 + \frac{2 \sigma}{\sqrt{n}})\sqrt{d}} \\
                        &= \; \frac{2 \sqrt{n} - 1}{2 \sqrt{n} + 4 \sigma} \sqrt{d}
  \end{align*}
  as stated in the lemma.
\end{proof}

\begin{lemma}
  \label{lem:gaussian_classification}
  Let $\vz \in \R^d$ be drawn from a spherical Gaussian, i.e., $\vz \sim \normal_d(\vmu, \sigma^2 \matI)$ where $\vmu \in \R^d$ and $\sigma > 0$.
  Moreover, let $\vw \in \R^d$ be an arbitrary unit vector with $\ip{\vw, \vmu} \geq \rho$ where $\rho \geq 0$.
  Then we have
  \[
    \prob\brackets*{\ip{\vw, \vz} \, \leq \, \rho} \; \leq \; \exp\parens*{-\frac{(\ip{\vw, \vmu} - \rho)^2}{2 \sigma^2}} \; .
  \]
\end{lemma}
\begin{proof}
  Since $\vz$ has the same distribution as $\vmu + \vg$ where $\vg \sim \normal_d(0, \sigma^2 \matI)$, we can bound the tail event as
  \begin{align*}
    \prob\brackets*{\ip{\vw, \vz} \, \leq \, \rho} \; &= \; \prob\brackets*{\ip{\vw, \vmu + \vg} \, \leq \, \rho} \\
                                                   &= \; \prob\brackets*{ \ip{\vw, \vg} \, \leq \, \rho - \ip{\vw, \vmu} } \; .
  \end{align*}
  The inner product $\ip{\vw, \vg}$ is distributed as a univariate normal $\normal(0, \sigma^2)$ because the vector $\vw$ has unit norm.
  Hence we can invoke the standard sub-Gaussian tail bound to get the desired tail probability.
\end{proof}

\begin{theorem}[Standard generalization in the Gaussian model.]
  \label{thm:gaussian_standard}
  Let $(\vx_1, y_1), \ldots, (\vx_n, y_n) \in \R^d \times \pmset$ be drawn i.i.d.\ from a $(\thetastar, \sigma)$-Gaussian model with $\norm{\thetastar}_2 = \sqrt{d}$.
  Let $\what \in\R^d$ be the unit vector in the direction of $\vzbar = \frac{1}{n} \sum_{i=1}^{n} y_i \vx_i$, i.e., $\what = \sfrac{\vzbar}{\norm{\vzbar}_2}$.
  Then with probability at least $1 - 2\exp(-\frac{d}{8 (\sigma^2 + 1)})$, the linear classifier $f_{\what}$ has classification error at most
  \[
    \exp\parens*{- \frac{(2 \sqrt{n} - 1)^2 d}{2 (2 \sqrt{n} + 4 \sigma)^2 \sigma^2}} \; .
  \]
\end{theorem}
\begin{proof}
  Let $\vz_i = y_i \cdot \vx_i$ and note that each $\vz_i$ is independent and has distribution $\normal_d(\thetastar, \sigma^2 \matI)$.
  Hence we can invoke Lemma \ref{lem:unit_ip} and get
  \[
      \ip{\what, \thetastar} \; \geq \; \frac{2 \sqrt{n} - 1}{2 \sqrt{n} + 4 \sigma} \sqrt{d}
  \]
  with probability at least $1 - 2\exp(-\frac{d}{8 (\sigma^2 + 1)})$ as stated in the theorem.

  Next, unwrapping the definition of $f_{\what}$ allows us to write the classification error of $f_{\what}$ as
  \[
    \prob\brackets*{f_{\what}(x) \neq y} \; = \; \prob\brackets*{\ip{\what, \thetastar} \leq 0} \; .
  \]
  Invoking Lemma \ref{lem:gaussian_classification} with $\rho = 0$ then gives the desired bound.
\end{proof}

\begin{corollary}[Generalization from a single sample.]
  \label{cor:single_sample}
  Let $(\vx, y)$ be drawn from a $(\thetastar, \sigma)$-Gaussian model with 
  \[
    \sigma \; \leq \; \frac{d^{\sfrac{1}{4}}}{5 \sqrt{\log\sfrac{1}{\beta}}} \; .
  \]
  Let $\what \in\R^d$ be the unit vector $\what = \frac{y \vx}{\norm{\vx}_2}$.
  Then with probability at least $1 - 2\exp(-\frac{d}{8 (\sigma^2 + 1)})$, the linear classifier $f_{\what}$ has classification error at most $\beta$.
\end{corollary}
\begin{proof}
  Invoking Theorem \ref{thm:gaussian_standard} with $n = 1$ gives a classification error bound of
  \[
    \beta' \; \eqdef \; \exp\parens*{- \frac{d}{2 (2 + 4\sigma)^2 \sigma^2}} \; .
  \]
  It remains to show that $\beta' \leq \beta$.

  We now bound the denominator in $\beta'$.
  First, we have
  \begin{align*}
    2 + 4 \sigma \; &\leq \; 2 d^{\sfrac{1}{4}} + \frac{4}{5} d^{\sfrac{1}{4}} \\
                    &\leq \; 3 d^{\sfrac{1}{4}} \; .
  \end{align*}
  Next, we bound the entire denominator as
  \begin{align*}
    2 (2 + 4\sigma)^2 \sigma^2 \; &\leq \; 2 \cdot 9 \sqrt{d} \cdot \frac{\sqrt{d}}{25 \log \sfrac{1}{\beta}} \\
                                  &\leq \; \frac{d}{\log \sfrac{1}{\beta}}
  \end{align*}
  which yields the desired classification error when substituted back into $\beta'$.
\end{proof}

\begin{lemma}
  \label{lem:robustupper}
Assume a $(\thetastar, \sigma)$-Gaussian model.
Let $p \geq 1$, $\eps \geq 0$ be robustness parameters, and let $\what$ be a unit vector such that $\ip{\what, \thetastar} \geq \eps \norm{\what}_p^*$., where $\norm{\cdot}_p^*$ is the dual norm of $\norm{\cdot}_p$.
Then the linear classifier $f_{\what}$ has $\ell_p^\eps$-robust classification error at most
\[
  \exp\parens*{-\frac{\parens{\ip{\what, \thetastar} - \eps \norm{\what}_p^*}^2}{2 \sigma^2}} \; .
\]
\end{lemma}
\begin{proof}
  Per Definition \ref{def:robusterror}, we have to upper bound the quantity
  \[
      \probop_{(x,y) \sim \distP}\brackets*{\, \exists \, x' \in \perturbB(x) \, : \: f_{\what}(x') \ne y } \; .
  \]
  For linear classifiers, we can rewrite this event as follows:
  \begin{align*}
    \probop_{(x,y) \sim \distP}\brackets*{\, \exists \, x' \in \perturbB_p^\eps(x) \, : \: f_{\what}(x') \ne y } \; &= \; \probop_{(x,y) \sim \distP}\brackets*{\, \exists \, x' \in \perturbB_p^\eps(x) \, : \: \ip{y \cdot x', \what} \leq 0 } \\
                                                                                                             &= \; \probop_{(x,y) \sim \distP}\brackets*{\, \exists \, \Delta  \in \perturbB_p^\eps(0) \, : \: \ip{y \cdot (x + \Delta), \what} \leq 0 } \\
    &= \; \probop_{(x,y) \sim \distP}\brackets*{\, \min_{\Delta  \in \perturbB_p^\eps(0)} \ip{y \cdot (x + \Delta), \what} \leq 0 } \\
    &= \; \probop_{(x,y) \sim \distP}\brackets*{\, \ip{y \cdot x, \what} + \min_{\Delta  \in \perturbB_p^\eps(0)} \ip{y \cdot \Delta, \what} \leq 0 } \; .
  \end{align*}
  We now use the definition of the dual norm.
  Note that for any $\Delta \in \perturbB_p^\eps$, we also have $-\Delta \in \perturbB_p^\eps$.
  Since $y \in \{\pm 1\}$, we can drop the $y$ factor.
  Overall, we get
  \begin{align*}
      \probop_{(x,y) \sim \distP}\brackets*{\, \ip{y \cdot x, \what} + \min_{\Delta  \in \perturbB_p^\eps(0)} \ip{y \cdot \Delta, \what} \leq 0 }  \; &= \; \probop_{(x,y) \sim \distP}\brackets*{\, \ip{y \cdot x, \what} - \eps \norm{\what}^*_p \leq 0 } \\
                                                  &= \; \probop_{(x,y) \sim \distP}\brackets*{\, \ip{y \cdot x, \what} \leq \eps \norm{\what}^*_p  }  \; .
        &
  \end{align*}
  By assumption in the lemma, we have $\ip{\what, \thetastar} \geq \eps \norm{\what}_p^*$.
  Hence we can invoke Lemma \ref{lem:gaussian_classification} with $\vmu = \thetastar$ and $\rho = \eps \norm{\what}_p^*$ to get the desired bound on the robust classification error.
\end{proof}

\begin{theorem}
  \label{thm:gausslinf}
  Let $(\vx_1, y_1), \ldots, (\vx_n, y_n) \in \R^d \times \pmset$ be drawn i.i.d.\ from a $(\thetastar, \sigma)$-Gaussian model with $\norm{\thetastar}_2 = \sqrt{d}$.
  Let $\what \in\R^d$ be the unit vector in the direction of $\vzbar = \frac{1}{n} \sum_{i=1}^{n} y_i \vx_i$, i.e., $\what = \sfrac{\vzbar}{\norm{\vzbar}_2}$.
  Then with probability at least $1 - 2\exp(-\frac{d}{8 (\sigma^2 + 1)})$, the linear classifier $f_{\what}$ has $\ell_\infty^\eps$-robust classification error at most $\beta$ if
  \[
    \eps \; \leq \; \frac{2 \sqrt{n} - 1}{2 \sqrt{n} + 4\sigma} - \frac{\sigma \sqrt{2 \log \sfrac{1}{\beta}}}{\sqrt{d}}  \; .
  \]
\end{theorem}
\begin{proof}
  Let $\vz_i = y_i \cdot \vx_i$ and note that each $\vz_i$ is independent and has distribution $\normal_d(\thetastar, \sigma^2 \matI)$.
  Hence we can invoke Lemma \ref{lem:unit_ip} and get
  \[
      \ip{\what, \thetastar} \; \geq \; \frac{2 \sqrt{n} - 1}{2 \sqrt{n} + 4 \sigma} \sqrt{d}
  \]
  with probability at least $1 - 2\exp(-\frac{d}{8 (\sigma^2 + 1)})$ as stated in the theorem.

  Since $\norm{\what}_2 = 1$, we have $\norm{\what}_\infty^* = \norm{\what}_1 \leq \sqrt{d}$.
  The bound on $\eps$ in the theorem allows us to invoke Lemma \ref{lem:robustupper}.
  This yields an $\ell_2^\eps$-robust classification error of at most
  \[
    \exp\parens*{-\frac{\parens{\ip{\what, \thetastar} - \eps \sqrt{d}}^2}{2 \sigma^2}} \; .
  \]
  Since
  \[
  \ip{\what, \thetastar} - \eps \norm{\what}_p^* \; \geq \; \frac{2 \sqrt{n} - 1}{2 \sqrt{n} + 4 \sigma} \sqrt{d} - \sqrt{d} \parens*{ \frac{2 \sqrt{n} - 1}{2 \sqrt{n} + 4\sigma} - \frac{\sigma \sqrt{2 \log \sfrac{1}{\beta}}}{\sqrt{d}}}
  \]
  this simplifies to the robust classification error stated in the theorem.
\end{proof}

\begin{corollary}
  \label{cor:gausslinf}
  Let $(\vx_1, y_1), \ldots, (\vx_n, y_n) \in \R^d \times \pmset$ be drawn i.i.d.\ from a $(\thetastar, \sigma)$-Gaussian model with $\norm{\thetastar}_2 = \sqrt{d}$ and $\sigma \leq \frac{1}{32} d^{\sfrac{1}{4}}$.
  Let $\what \in\R^d$ be the unit vector in the direction of $\vzbar = \frac{1}{n} \sum_{i=1}^{n} y_i \vx_i$, i.e., $\what = \sfrac{\vzbar}{\norm{\vzbar}_2}$.
  Then with probability at least $1 - 2\exp(-\frac{d}{8 (\sigma^2 + 1)})$, the linear classifier $f_{\what}$ has $\ell_\infty^\eps$-robust classification error at most $0.01$ if
  \[
  n \; \geq \; \begin{cases} 1 \quad &\text{ for } \;\; \eps \, \leq \, \frac{1}{4}d^{-\sfrac{1}{4}} \\
  64 \, \eps^2\sqrt{d} & \text{ for } \; \; \frac{1}{4}d^{-\sfrac{1}{4}} \, \leq \, \eps \, \leq \, \frac{1}{4}\end{cases} \; .
  \]
\end{corollary}
\begin{proof}
  We begin by invoking Theorem \ref{thm:gausslinf}, which gives a $\ell_\infty^{\eps'}$-robust classification error at most $\beta = 0.01$ for
  \begin{align*}
    \eps' \; &= \; \frac{2 \sqrt{n} - 1}{2 \sqrt{n} + 4\sigma} - \frac{\sigma \sqrt{2 \log \sfrac{1}{\beta}}}{\sqrt{d}}  \\
             &\geq \; \frac{2 \sqrt{n} - 1}{2 \sqrt{n} + \frac{1}{8}d^{\sfrac{1}{4}}} - \frac{1}{8 d^{\sfrac{1}{4}}} \; .
  \end{align*}

  First, we consider the case where $\eps \leq \frac{1}{4} d^{-\sfrac{1}{4}}$.
  Using $n = 1$, the resulting robustness is
  \begin{align*}
        \eps' &\geq \; \frac{1}{2 + \frac{1}{8}d^{\sfrac{1}{4}}} - \frac{1}{8 d^{\sfrac{1}{4}}} \\
              &\geq \; \frac{1}{(2 + \frac{1}{8})d^{\sfrac{1}{4}}} - \frac{1}{8 d^{\sfrac{1}{4}}} \\
              &\geq \; \frac{1}{4} d^{-\sfrac{1}{4}} \\
              &\geq \; \eps
  \end{align*}
  as required.

  Next, we consider the case $\frac{1}{4}d^{-\sfrac{1}{4}} \, \leq \, \eps \, \leq \, \frac{1}{4}$.
  Substituting $n = 64\eps^2 \sqrt{d}$, we get
  \begin{align*}
    \eps' &\geq \; \frac{16 \eps d^{\sfrac{1}{4}} - 1}{16 \eps d^{\sfrac{1}{4}} + \frac{1}{8} d^{\sfrac{1}{4}}} - \frac{1}{8 d^{\sfrac{1}{4}}} \\
         &\geq \; \frac{12 \eps d^{\sfrac{1}{4}}}{4 d^{\sfrac{1}{4}} + \frac{1}{8} d^{\sfrac{1}{4}}} - \frac{1}{8 d^{\sfrac{1}{4}}} \\
         &\geq \; \frac{12}{5} \eps - \frac{1}{2} \eps \\
         &\geq \; \eps
  \end{align*}
  which completes the proof.
\end{proof}

\subsection{Lower bound}
\label{app:gaussian_lb}
The following theorem is our main lower bound for the Gaussian model.
To make the lower bound easily comparable to Corollary \ref{cor:gausslinf} on the upper bound side, we simplify the lower bound in Corollary \ref{cor:gaussian_robust} and bring it into a similar form.

\gaussianlbmain*
\begin{proof}
  We begin by formally stating the expected $\ell_\infty^\eps$-robust classification error of $f_n$:
  \[
    \Xi \; = \; \Eop_{\theta \sim \normal(0, \matI)} \brackets*{ \Eop_{y_1, \ldots, y_n \sim \rademacher} \brackets*{ \Eop_{\substack{x_1, \ldots, x_n \\ \sim \normal(y_i \theta, \sigma^2 \matI)}} \brackets*{ \Eop_{y \sim\rademacher } \brackets*{ \probop_{x \sim \normal(y \theta, \sigma^2 \matI) } \brackets*{\, \exists \, x' \in \perturbB_\infty^\eps(x) \, : \: f_n(x') \ne y } }  } } }
  \]
  where it is important to note that $f_n = g_n((x_1, y_1), \ldots, (x_n, y_n))$ depends on the samples $(x_i, y_i)$ but not on $\theta$.
  This will allow us to re-arrange the above expectations in a crucial way.

  We first rewrite the expectations by noting that we can sample $z_i \sim \normal(\theta, \sigma^2 \matI)$ without conditioning on the class $y_i$ by then setting $f_n = g_n((y_1 z_1, y_1), \ldots, (y_n z_n, y_n))$.
  This yields
  \begin{align*}
    \Xi \; &= \; \Eop_{\theta \sim \normal(0, \matI)} \brackets*{ \Eop_{y_1, \ldots, y_n \sim \rademacher} \brackets*{ \Eop_{\substack{z_1, \ldots, z_n \\ \sim \normal(\theta, \sigma^2 \matI)}} \brackets*{ \Eop_{y \sim \rademacher} \brackets* { \probop_{x \sim \normal(y \theta, \sigma^2 \matI) } \brackets*{\, \exists \, x' \in \perturbB_\infty^\eps(x) \, : \: f_n(x') \ne y } }  } } } \\
           &= \; \Eop_{y_1, \ldots, y_n \sim \rademacher} \brackets*{ \Eop_{\theta \sim \normal(0, \matI)} \brackets*{  \Eop_{\substack{z_1, \ldots, z_n \\ \sim \normal(\theta, \sigma^2 \matI)}} \brackets*{ \Eop_{y \sim \rademacher } \brackets* {\probop_{x \sim \normal(y \theta, \sigma^2 \matI) } \brackets*{\, \exists \, x' \in \perturbB_\infty^\eps(x) \, : \: f_n(x') \ne y } }  } } }
  \end{align*}
  where in the second line we moved the expectation over the class labels to the outside.
  
  Next, we will swap the order of the expectations over the mean parameter $\theta$ and the conditional samples $x_i$.
  Since the posterior distribution for a Gaussian prior and likelihood is also Gaussian, the conditional distribution of $\theta$ given the $z_i$ is a multivariate Gaussian with parameters
  \begin{align*}
    \vmu' \; & = \; \frac{n}{\sigma^2 + n} \vzbar \\
    \Sigma' \; & = \; \frac{\sigma^2}{\sigma^2 + n} \, \matI
  \end{align*}
  where $\vzbar = \sum_{i=1}^n z_i$.
  Moreover, let $\M$ be the marginal distribution over $(\vz_i, \ldots, \vz_n)$ after integrating over $\theta$ (which we will analyze later).
  Then we get
  \begin{equation}
    \label{eq:gauss_linf_psi}
    \Xi \; = \; \Eop_{y_1, \ldots, y_n \sim \rademacher} \brackets*{  \Eop_{(z_1, \ldots, z_n) \sim \M} \brackets*{ \underbrace{ \Eop_{\theta \sim \normal(\vmu', \Sigma')} \brackets*{ \Eop_{y \sim \rademacher} \brackets*{ \probop_{x \sim \normal(y \theta, \sigma^2 \matI) } \brackets*{\, \exists \, x' \in \perturbB_\infty^\eps(x) \, : \: f_n(x') \ne y } } } }_{=\Psi}  } }
  \end{equation}
  We now bound the term $\Psi$.
  Since the inner events only depends on $\theta$ through $x$, we can combine the Gaussian expectation with the Gaussian probability after moving the expectation over the label $y$ to the outside.
  This gives
  \begin{align*}
    \Psi \; &= \; \Eop_{\theta \sim \normal(\vmu', \Sigma')} \brackets*{ \Eop_{y \sim \rademacher} \brackets* {\probop_{x \sim \normal(y \theta, \sigma^2 \matI) } \brackets*{\, \exists \, x' \in \perturbB_\infty^\eps(x) \, : \: f_n(x') \ne y } } } \\
            &= \; \Eop_{y \sim \rademacher} \brackets*{ \Eop_{\theta \sim \normal(\vmu', \Sigma')} \brackets*{ {\probop_{x \sim \normal(y \theta, \sigma^2 \matI) } \brackets*{\, \exists \, x' \in \perturbB_\infty^\eps(x) \, : \: f_n(x') \ne y } } } } \\
            &= \;  \Eop_{y \sim \rademacher} \brackets*{ \probop_{x \sim \normal(y \vmu', \Sigma'') } \brackets*{\, \exists \, x' \in \perturbB_\infty^\eps(x) \, : \: f_n(x') \ne y } } \numberthis \label{eq:gauss_linf_max}
  \end{align*}
  where $\Sigma'' = \Sigma' + \sigma^2 \matI$.

  Next, we bound the $y = +1$ case in the expectation over $y$.
  The $y = -1$ case can be handled exactly analogously.
  We introduce the set $A_{-} \subseteq \R^d$ as the set of inputs on which the classifier $f_n$ returns $-1$, i.e., $A_{-} = \{ x \, | \, f_n(x) = -1\}$.
  Note that we can treat $A_{-}$ as fixed here since it only depends on the samples $z_i$ and labels $y_i$ but not on the parameter $\theta$ or the new sample $x$.
  This allows us to rewrite the first event as
  \begin{align*}
    \{ x \, | \, \exists \, x' \in \perturbB_\infty^\eps(x) \, : \: f_n(x') \ne +1 \} \; &= \; \{ x \, | \, \exists \, x' \in A_{-} \, : \, \norm{x - x'}_\infty \leq \eps \} \\
                                                                                         &= \; \perturbB_\infty^\eps(A_{-}) \; .
  \end{align*}
  Now, note that as long as $\norm{\mu'}_\infty \leq \eps$, the set $\perturbB_\infty^\eps(A_{-})$ contains a copy of $A_{-}$ shifted by $\pm \mu'$.
  Hence we have
  \begin{align*}
    \probop_{x \sim \normal(\vmu', \Sigma'') } \brackets*{\, \exists \, x' \in \perturbB_\infty^\eps(x) \, : \: f_n(x') \ne +1 } \; &= \; \probop_{\normal(\vmu', \Sigma'') } \brackets*{ \perturbB_\infty^\eps(A_{-}) } \\
                                                                                                                                      &\geq \; \I \brackets*{ \norm{\mu'}_\infty \leq \eps } \cdot \probop_{\normal(0, \Sigma'')}  \brackets*{ A_{-} }
  \end{align*}
  Repeating the same argument for the $y = -1$ case and substituting back into Equation \eqref{eq:gauss_linf_max} yields
  \begin{align*}
    \Psi \; &\geq \; \Eop_{y \sim \rademacher} \brackets*{ \I \brackets*{ \norm{\mu'}_\infty \leq \eps }  \cdot \probop_{\normal(0, \Sigma'')}  \brackets*{ A_{-\sign(y)} } } \\
          &= \; \I \brackets*{ \norm{\mu'}_\infty \leq \eps }  \cdot \frac{1}{2} \parens*{ \probop_{\normal(0, \Sigma'')} \brackets*{ A_{-}} \, + \,  \probop_{\normal(0, \Sigma'')} \brackets*{ A_{+}}} \\
          &= \; \frac{1}{2} \I \brackets*{ \norm{\mu'}_\infty \leq \eps }   \; .
  \end{align*}
  In the last line, we used that the sets two sets $A_{-}$ and $A_{+}$ are complements of each other and hence their total mass under the measure $\normal(0, \Sigma'')$ is 1.

  Substituting back into Equation \eqref{eq:gauss_linf_psi} yields
  \begin{align*}
    \Xi \; &\geq \; \Eop_{y_1, \ldots, y_n \sim \rademacher} \brackets*{  \Eop_{(z_1, \ldots, z_n) \sim \M} \brackets*{ \frac{1}{2} \I \brackets*{ \norm{\mu'}_\infty \leq \eps }  }  } \\
    \; &= \; \frac{1}{2} \Eop_{(z_1, \ldots, z_n) \sim \M} \brackets*{ \I \brackets*{ \norm{\mu'}_\infty \leq \eps } } \\
    \; &= \; \frac{1}{2} \probop_{(z_1, \ldots, z_n) \sim \M} \brackets*{ \frac{n}{\sigma^2 + n}\norm{\vzbar}_\infty \leq \eps }
  \end{align*}
  where we dropped the expectation over the labels $y_i$ since the inner expression is now independent of the labels.
  
  It remains to analyze the distribution of the vector $\vzbar$.
  Note that conditioned on a vector $\theta_2 \sim \normal_d(0, \matI)$, the distribution of each $z_i$ is $\normal(\theta_2, \sigma^2 \matI)$.
  Hence the conditional distribution of $\vzbar$ given $\theta_2$ is $\normal(\theta_2, \frac{\sigma^2}{n} \matI)$ and integrating over $\theta_2$ yields a marginal distribution of $\normal(0, (1 + \frac{\sigma^2}{n})\matI)$.
  Overall, this gives
  \begin{align*}
    \Xi \; &\geq \;  \frac{1}{2} \probop_{\theta_2 \sim \normal(0, (1 + \frac{\sigma^2}{n})\matI)} \brackets*{ \frac{n}{\sigma^2 + n}\norm{\theta_2}_\infty \leq \eps   } \\
           &= \;  \frac{1}{2} \probop_{\theta_2 \sim \normal(0, \matI)} \brackets*{ \sqrt{\frac{n}{\sigma^2 + n}} \norm{\theta_2}_\infty \leq \eps   }
  \end{align*}
  where we used
  \[
    \frac{n}{\sigma^2 + n} \sqrt{1 +  \frac{\sigma^2}{n} } \; = \;  \sqrt{\frac{n}{\sigma^2 + n}} \; .
  \]
  Rearranging this inequality yields the statement of the theorem.
\end{proof}

\begin{corollary}
  \label{cor:gaussian_robust}
  Let $g_n$ be any learning algorithm, i.e., a function from $n \geq 0$ samples in $\R^d \times \{\pm 1\}$ to a binary classifier $f_n$.
  Moreover, let $\sigma > 0$, let $\eps \geq 0$, and let $\theta \in \R^d$ be drawn from $\normal(0, \matI)$.
  We also draw $n$ samples from the $(\theta, \sigma)$-Gaussian model.
  Then the expected $\ell_\infty^\eps$-robust classification error of $f_n$ is at least $(1 - \sfrac{1}{d}) \frac{1}{2}$ if
  \[
    n \; \leq \; \frac{\eps^2 \sigma^2}{8 \log d} \; .
  \]
\end{corollary}
\begin{proof}
  We have
  \[
    \sqrt{\frac{n}{\sigma^2 + n}} \; \leq \; \sqrt{\frac{\eps^2 \sigma^2}{8 \sigma^2 \log d}} \; = \; \frac{\eps}{2 \sqrt{2 \log d}} \; .
  \]
  Hence we get
  \begin{align*}
    \probop_{v \sim \normal(0, \matI)} \brackets*{ \sqrt{\frac{n}{\sigma^2 + n}} \norm{v}_\infty \leq \eps } \; &\geq \; \probop_{v \sim \normal(0, \matI)} \brackets*{ \sqrt{\frac{\eps}{2 \sqrt{2 \log d}}} \norm{v}_\infty \leq \eps } \\
            &= \; \probop_{v \sim \normal(0, \matI)} \brackets*{  \norm{v}_\infty \leq 2 \sqrt{2 \log d}}  \; .
  \end{align*}
  Standard concentration results for the maximum of $d$ i.i.d.\ Gaussians (e.g., see Theorem 5.8 in \citep{BLM2013}) now imply that the above probability is at least $(1 - \sfrac{1}{d})$.
  Invoking into Theorem \ref{thm:gauss_linf_lower} then completes the proof of this corollary.
\end{proof}

\section{Omitted proofs for the Bernoulli model}
\label{app:bernoulli}

\subsection{Upper bounds}
\label{app:bernoulli_upper}
As in the Gaussian case, our upper bounds rely on standard sub-Gaussian concentration.
Lemmas \ref{lem:bernoulli_ip1} and \ref{lem:bernoulli_unit_ip} provide lower bounds on the inner product between a single sample from the Bernoulli model and the unknown parameter vectors.
Lemma \ref{lem:bernoulli_classification} then relates the inner product between a linear classifier and the unknown mean vector to the classification accuracy.
Combining these results yields Theorem \ref{thm:bernoulli_standard} for generalization from a single sample.
Simplifying this theorem yields Corollary \ref{cor:bernoulli_single_sample}, which directly implies Theorem \ref{thm:bernoulli_standard_upper} from the main text.

\begin{lemma}
	\label{lem:bernoulli_ip1}
	Let $(\vx, y) \in \R^d \times \pmset$ be a sample drawn from a $(\thetastar, \tau)$-Bernoulli model and let $\vz = \vx y$.
	Let $\delta > 0$ be the target probability.
	Then we have
	\[
\prob \brackets*{ \ip{\vz, \thetastar} \, \leq \, 2 \tau d  -   \sqrt{{ 2 d\log \sfrac{1}{\delta}}} \, } \, \leq \, \delta \; .
	\]
\end{lemma}
\begin{proof}
	To center $\vz$, we define $\vg = \vz - \mathbb{E}[\vz] = \vz - 2 \tau \thetastar$, where each coordinate of $\vg$ has zero mean. Then, we can write
  \[
    \ip{\vz, \thetastar} = \ip{\vg + 2\tau \thetastar, \thetastar} =\ip{\vg, \thetastar} +  2 \tau d \; .
  \]
  Hence for all $t > 0$ we have,
	\[
	\prob \brackets*{ \ip{\vz, \thetastar} \, \leq \, 2 \tau d  -   t \, } \, = \, \prob \brackets*{ \ip{\vg, \thetastar} \, \leq \, -   t \, } \; .
	\]
  Note that $\vg = (\vg_1, \vg_2, ..., \vg_d)$ is a vector of sub-Gaussian random variables since each entry is bounded, i.e., each $\vg_j$ (like $\vz_j$) lies in an interval of length $2$.
  Hence, the sub-Gaussian parameter of each $\vg_j$ is 1.
 Invoking Corollary 1.7 from \citet{rigollet2015high} for the weighted combination of independent sub-Gaussian random variables, we get that 	
	\[
	\prob \brackets*{ \sum_{j=1}^d \vg_j \thetastar_j \, \leq \,  -   t \, } \, \leq \, \exp \parens*{-\frac{t^2}{2 ||\thetastar||_2^2}} \,  \; .
	\]	
  Since $\norm{\thetastar}_2^2 = d$, we can simplify the tail event 
	\[
	\prob \brackets*{ \ip{\vz, \thetastar} \leq  2 \tau d  - t } \; \leq \;   \exp \parens*{-\frac{ t^2}{2d}} 
	\]
which then gives
\begin{align*}
	& \prob \brackets*{ \ip{\vz, \thetastar} \, \leq \, 2 \tau d -   \sqrt{{2 d \log \sfrac{1}{\delta}}} \, } \, \leq \, \delta
\end{align*}
as desired.
\end{proof}

\begin{lemma}
	\label{lem:bernoulli_unit_ip}
	Let $(\vx, y) \in \R^d \times \pmset$ be a sample drawn from a $(\thetastar, \tau)$-Bernoulli model and let $\vz = \vx y$.  Let $\what \in \R^d$ be the unit vector in the direction of $\vz$, i.e., $\what = \sfrac{\vz}{\norm{\vz}_2}$.
	Then we have
	\[
	\prob\brackets*{ \ip{\what, \thetastar} \, \leq \, \tau \sqrt{d} } \; \leq \; \exp\parens*{-\frac{\tau^2 d}{2}} \; .
	\]
\end{lemma}
\begin{proof}
	We know that
	\[
	\norm{\vz}_2 \, = \, \sqrt{d} \; .
	\]
	Moreover, we invoke Lemma \ref{lem:bernoulli_ip1} with $\delta = \exp\parens{-\frac{\tau^2 d}{2}}$ to get
	\[
	\ip{\vz, \thetastar} \, \leq \, \tau d 
	\]
	with probability $\delta$.
	We now have
	\begin{align*}
	\ip{\what, \thetastar} \; &= \; \frac{\ip{\vz, \thetastar}}{ \norm{\vz}_2} \\
	&\leq \; \frac{{\tau d}}{\sqrt{d}} 
	\end{align*}
	with probability $\delta$ as stated in the lemma. 
\end{proof}

\begin{lemma}
	\label{lem:bernoulli_classification}
	Let $(\vx, y) \in \R^d \times \pmset$ be a sample drawn from a $(\thetastar, \tau)$-Bernoulli model and let $\vz = \vx y$.
	Moreover, let $\vw \in \R^d$ be an arbitrary unit vector with $\ip{\vw, 2 \tau \thetastar} \geq 0$.
	Then we have
	\[
	\prob\brackets*{\ip{\vw, \vz} \, \leq \, 0} \; \leq \; \exp\parens*{-2 \tau^2 \ip{\vw, \thetastar}^2} \; .
	\]
\end{lemma}
\begin{proof}
  As in Lemma~\ref{lem:bernoulli_ip1}, we center $\vz = 2 \tau \thetastar + \vg$, where $\vg$ is a vector of zero-mean sub-Gaussian random variables. We can bound the tail event as 
	\begin{align*}
	\prob\brackets*{\ip{\vw, \vz} \, \leq \, 0} \; &= \; \prob\brackets*{\ip{\vw, 2 \tau \thetastar + \vg} \, \leq \, 0} \\
	&= \; \prob\brackets*{ \ip{\vw, \vg} \, \leq \, - \ip{\vw, 2 \tau \thetastar} } \; .
	\end{align*}
	 We know that the sub-Gaussian parameter of each $\vg_j$ is 1 as discussed in Lemma~\ref{lem:bernoulli_ip1}. Hence, invoking Corollary 1.7 from \citet{rigollet2015high} for the weighted combination of independent sub-gaussian random variables, we get that 	
	\[
	\prob \brackets*{ \sum_{j=1}^d \vg_j w_j \, \leq \,  -   t \, } \, \leq \, \exp \parens*{-\frac{t^2}{2 ||w||_2^2}} \,  =  \exp \parens*{-\frac{t^2}{2}} \; .
	\]	
	Thus, 
	\[
	\prob\brackets*{ \ip{\vw, \vg} \, \leq \, - \ip{\vw, 2 \tau \thetastar} } \; \leq \;  \exp \parens*{-\frac{\ip{\vw, 2 \tau \thetastar}^2}{2}}
\]
  as desired in the lemma.
\end{proof}
\begin{theorem}[Standard generalization in the Bernoulli model.]
	\label{thm:bernoulli_standard}
	Let $(\vx, y) \in \R^d \times \pmset$ be drawn from a $(\thetastar, \tau)$-Bernoulli model.
	Let $\what \in\R^d$ be the unit vector in the direction of $\vz = y \vx$, i.e., $\what = \sfrac{\vz}{\norm{\vz}_2}$.
	Then with probability at least $1 - \exp(-\frac{\tau^2d}{2})$, the linear classifier $f_{\what}$ has classification error at most $\exp\parens*{-{2 \tau^4d}}$.
\end{theorem}
\begin{proof}
	We invoke Lemma \ref{lem:bernoulli_unit_ip} to get
	\[
	\ip{\what, \thetastar} \; \geq \; \tau \sqrt{d}
	\]
	with probability at least $1 - \exp(-\frac{\tau^2 d}{2})$ as stated in the theorem.  Next, unwrapping the definition of $f_{\what}$ allows us to write the classification error of $f_{\what}$ as
	\[
	\prob\brackets*{f_{\what}(x) \neq y} \; = \; \prob\brackets*{\ip{\what, \vz} \leq 0} \; .
	\]
	Invoking Lemma \ref{lem:bernoulli_classification} then gives the desired bound.
\end{proof}

\begin{corollary}[Generalization from a single sample.]
  \label{cor:bernoulli_single_sample}
	Let $(\vx, y) \in \R^d \times \pmset$ be drawn from a $(\thetastar, \tau)$-Bernoulli model with 
	\[
	\tau \; \geq \; \parens*{\frac{\log\sfrac{1}{\beta}}{2 d}}^{\sfrac{1}{4}} \; .
	\]
	Let $\what \in\R^d$ be the unit vector $\what = \frac{y \vx}{\norm{\vx}_2}$.
	Then with probability at least $1 - \exp(-\frac{\tau^2d}{2})$, the linear classifier $f_{\what}$ has classification error at most $\beta$.
\end{corollary}
\begin{proof}
	Invoking Theorem \ref{thm:bernoulli_standard} gives a classification error bound of
	\[
	\beta' \; \eqdef \; \exp\parens*{- {2 \tau^4d}} \; .
	\]
	It remains to show that $\beta' \leq \beta$. Now,
	\begin{align*}
	\log \sfrac{1}{\beta'} &= {2 \tau^4 d} \\
	&\geq \; 2 \cdot \frac{\log\sfrac{1}{\beta}}{2 d} \cdot {d} \\
	&\geq \; \log\sfrac{1}{\beta}
	\end{align*}
	which yields the desired bound.
\end{proof}

\subsection{Lower bounds}
\label{app:bern_lb}

In this section, we show that any \emph{linear classifier} for the $(\theta^*, \tau)$-Bernoulli model requires many samples to be robust.
The main result is formalized in Theorem \ref{thm:bern_lin_lb}, which can be simplified to yield Theorem \ref{thm:main_bernoulli_lower} from the main text.
Before we proceed to the main theorem, we first prove a simple but useful lemma.
\begin{lemma}
\label{lem:bb_one_d}
Let $\theta$ be drawn uniformly at random from $\{-1, 1\}$ and let $(x_1, y_1), \ldots, (x_n, y_n)$ be drawn independently from the $(\theta, \tau)$-Bernoulli model. 
Then for $\tau\leq \sfrac{1}{4}$ and $n \leq \frac{1}{\tau^2}$, we have with probability $1 - \delta$ over the samples that
\begin{align*}
  \log \frac{\Pr[\theta =+1 \mid (x_1, y_1) \ldots, (x_n, y_n)]}{\Pr[\theta = -1 \mid (x_1, y_1) \ldots, (x_n, y_n)]} \; \in \; \brackets*{-15\tau\sqrt{2n\log \frac 2 \delta}, \; 15\tau\sqrt{2n\log \frac 2 \delta}}
\end{align*}
\end{lemma}
\begin{proof}
For any sequence $(x_1, y_1) \ldots, (x_n, y_n)$, we can write
\begin{equation}
  \frac{\Pr[\theta =+1 \mid (x_1, y_1) \ldots, (x_n, y_n)]}{\Pr[\theta =-1 \mid (x_1, y_1) \ldots, (x_n, y_n)]} \; = \; \frac{\Pr[(x_1, y_1),\ldots, (x_n, y_n) \mid \theta=+1]}{\Pr[(x_1, y_1),\ldots, (x_n, y_n) \mid \theta = -1]} \label{eq:bernlb_lem_1}
\end{equation}
because $\Pr[\theta = +1] = \Pr[\theta = -1]$.
We now simplify the right hand side to
\begin{align*}
\frac{\Pr[(x_1, y_1),\ldots, (x_n, y_n) \mid \theta=+1]}{\Pr[(x_1, y_1),\ldots, (x_n, y_n) \mid \theta = -1]} &= \; \prod_{i=1}^n \frac{\Pr[(x_i, y_i)  \mid \theta=+1]}{\Pr[ (x_i, y_i) \mid \theta = -1]}\\
                                                                                                              &= \; \prod_{i=1}^n \left(\frac{\frac{1}{2} + \tau}{\frac{1}{2} - \tau}\right)^{y_i x_i} \numberthis \label{eq:bernlb_lem_2}
\end{align*}
where the second line follows from a simple calculation of the conditional probabilities.

Writing $z_i = y_i x_i$, we next combine Equations \eqref{eq:bernlb_lem_1} and \eqref{eq:bernlb_lem_2} to
\[
\frac{\Pr[\theta =+1 \mid (x_1, y_1) \ldots, (x_n, y_n)]}{\Pr[\theta =-1 \mid (x_1, y_1) \ldots, (x_n, y_n)]} \; = \; \exp\parens*{\hat{\tau} \sum_{i=1}^n z_i} \; ,
\]
where $\hat{\tau}$ is such that $\exp(\hat{\tau}) = \frac{1 + 2\tau}{1 - 2\tau}$. For $\tau \leq \frac 1 4$, a simple calculation shows that $\hat{\tau} \leq 5\tau$.

Conditioned on $\theta$, the sum $\zbar = \sum_{i=1}^n z_i$ has expectation $2\tau n \theta \leq 2\tau n$.
Hoeffing's Inequality (e.g., see Theorem 2.8 in \cite{BLM2013}) then yields that with probability $1 - \sfrac{\delta}{2}$
\[
  \zbar \; \leq \; 2\tau n + \sqrt{2n \log \frac 2 \delta} \; .
\]
It follows that with probability $1 - \sfrac{\delta}{2}$ (taken over the samples $z_1,\ldots, z_n$), the likelihood ratio above is bounded by
\begin{align*}
  \exp(\hat{\tau} \sum_i z_i) &\leq \exp\parens*{2\hat{\tau} \tau n + \hat{\tau} \sqrt{2n\log \frac 2 \delta}}
\end{align*}
Under the assumptions that $n \leq \frac{1}{\tau^2}$, we have
\[
  \tau n \; \leq \; \sqrt{n}
\]
and the upper bound follows because the first term in the $\exp$ is at most twice the second term.
The lower bound is symmetric.
\end{proof}

We next evaluate the $\ell_{\infty}$ robustness of the optimal linear classifier.
\begin{lemma}
\label{lem:linear_besteps}
Let $\thetastar \in \{-1,+1\}^d$ and consider the linear classifier $f_{\thetastar}$ for the $(\thetastar, \tau)$-Bernoulli model. Then,
\begin{description}
  \item[$\ell_\infty^{\tau}$-robustness:] The $\ell_\infty^{\tau}$-classification error of $f_{\thetastar}$ is at most $2\exp(-\tau^2d/2)$.
  \item[$\ell_\infty^{3\tau}$-nonrobustness:] The $\ell_\infty^{3\tau}$-classification error of $f_{\thetastar}$ is at  least  $1-2\exp(-\tau^2d/2)$.
\item[Near-optimality of $\thetastar$:] For any linear classifier, the $\ell_\infty^{3\tau}$-classification error is at least $\frac 1 6$.
\end{description}
\end{lemma}
\begin{proof}
Let $(x,y)$ be drawn from the $(\thetastar, \tau)$-Bernoulli model.
Then for the linear classifier $w=\thetastar$, we have
\[
  \E[\ip{w, y x}] \; = \; 2\tau \ip{w, \thetastar} \; = \; 2\tau d \; .
\]
Let $S$ denote the set
\[
  S \, = \, \{(x,y) : \ip{w, y x} \in [\tau d, 3\tau d] \} \; .
\]
Hoeffing's Inequality (e.g., see Theorem 2.8 in \cite{BLM2013}) then gives
\begin{align*}
  \Pr[(x,y) \not\in S] \; = \; \Pr\left[\ip{w, y x} \not\in [\tau d, 3\tau d]\right] \; \leq \; 2\exp(-\tau^2 d/ 2) \; .
\end{align*}
On the other hand, for a parameter $\eps$,
\begin{align*}
  \sup_{e \in B_\infty^\eps} \ip{w, e} \; = \; \eps \norm{w}_1 \; = \; \eps d \; .
\end{align*}
Thus if $\eps < \tau$, then for any $(x,y) \in S$,
\begin{align*}
  \inf_{e \in B_{\infty}^\eps} \ip{w, y (x+e)} & \; > \; 0 \; ,
\end{align*}
so that any $(x,y) \in S$ is $\ell_\infty^{\tau}$-robustly classified.
On the other hand, for $\eps > 3\tau$, for any $(x,y) \in S$,
\begin{align*}
  \inf_{e \in B_{\infty}^\eps} \ip{w, y(x+e)} \; &< \; 0 \; ,
\end{align*}
so that $(x,y)$ is not $\ell_\infty^{3\tau}$-robustly classified.

Finally, let $w'$ be any other linear classifier.
Then we have
\begin{align*}
  \E[ \ip{w', y x}] \; = \; 2\tau \ip{w', \thetastar} \; \leq \; 2\tau \norm{w'}_1 \; ,
\end{align*}
Let $E_i$ be a $\pm 1$ random variable with expectation $2 \tau$.
We observe that the random variable $yx_iw'_i$ is stochastically dominated by $E_i \cdot \abs{w'_i}$ (note that $y x_i$ is itself a $\pm 1$ random variable with expectation $2 \tau$).
We can now write $E_i$ as
\[
  E \; = \; A_i + B_i \; ,
\]
where the random variable $A_i$ is in $\{0,1\}$ and has expectation $2\tau$.
The random variable $B_i$ is in $\{-1, 0, 1\}$ and has a symmetric distribution that depends on $A$.
In particular, $B_i = 0$ iff $A_i = 1$ and $B_i$ is a Rademacher random variable otherwise.
Since $A_i$ is non-negative, we can use Markov's inequality on $\sum_i \abs{w_i} A_i$.
The $B_i$'s have a symmetric distribution even conditioned on $A_i$ so that $\sum_i \abs{w'_i} B_i \leq 0$ with probability at least $\sfrac{1}{2}$.
Thus with probability at least $\sfrac{1}{6}$, we have
\[
  \ip{w', y x} \; \leq \; 3\tau \norm{w'}_1 \; .
\]
Thus for any $\eps > 3\tau$,
\begin{align*}
  \inf_{e \in B_\infty^{\eps}} \ip{w', y(x+e)} \; &= \; \ip{w', y x} + \inf_{e \in B_{\infty}^{\eps}}\ip{w', y e}\\
                                                  &\leq \; 3\tau \norm{w}_1 - \eps \norm{w}_1 \\
   & < \; 0 \; .
\end{align*}
Thus the $\ell_{\infty}^{3\tau}$-classification error of $w'$ is at least $\sfrac{1}{6}$.
\end{proof}

Lemma \ref{lem:linear_besteps} implies that the most interesting robustness regime for linear classifiers is $\eps = O(\tau)$.
For larger values of $\eps$, it is impossible to learn a linear classifier with small robust classification error regardless of the number of samples used.

We now focus on this robustness regime and establishes a lower bound on the sample complexity of $\ell_{\infty}^{\eps}$-robust classification for $\eps \in (0,\tau)$.
\begin{theorem}
  \label{thm:bern_lin_lb}
Let $g_n$ be a linear classifier learning algorithm, i.e., a function that takes $n$ samples from $\{-1, +1\}^d \times {\pm 1}$ to a linear classifier $w \in \R^d$.
Suppose that we choose $\thetastar$ uniformly at random from $\{ -1, +1\}^d$ and draw $n$ samples from the $(\thetastar, \tau)$-Bernoulli model with $\tau \leq \sfrac{1}{4}$.
Let $w$ then be the output of $g_n$ on these samples.
Moreover, let $\eps < 3 \tau$ and $0 < \gamma < \sfrac{1}{2}$.
Then if
\[
n \; \leq \; \frac{\eps^2\gamma^2}{5000 \cdot \tau^4\log (4d/\gamma)}
\]
the linear classifier $f_w$ has expected $\ell_{\infty}^{\eps}$-classification error at least $\frac{1}{2} - \gamma$.
\end{theorem}

Before we proceed to the formal proof, we briefly explain the approach at a high level.
Informally, Lemma~\ref{lem:bb_one_d} above implies that for small $n$, the algorithm $g_n$ is sufficiently uncertain about each co-ordinate $\thetastar_i$ so that in expectation, the dot product $\ip{w, \thetastar}$ is small compared to $\|w\|_1$.
Since the $\ell_1$ norm $\|w\|_1$ is dual to the $\ell_\infty$ norm bounding the adversarial perturbation, it can be related to the adversarial robustness of the classifier $w$ on a fresh sample $x$.
This then leads to the lower bound stated above, as we will now prove in more detail.
\begin{proof}
Let $S$ be the set of $n$ samples input to $g_n$ and let $w$ be the resulting classifier as defined in the theorem.
Our first goal is to bound the uncertainty in the estimate $w$ by establishing an upper bound on $\abs*{\E[\thetastar_i | S]}$ for each $i \in [d]$, which will in turn allow us to bound $\E_{\thetastar}\brackets*{\ip{w, \thetastar} | S}$.

We have
\[
  \E[ \thetastar_i | S ] \; = \; \prob\brackets{\thetastar_i = +1 | S} - \prob\brackets{\thetastar_i = -1 | S} \; .
\]
We first consider the case that $\prob\brackets{\thetastar_i = +1 | S} \geq \prob\brackets{\thetastar_i = -1 | S}$, which means that the conditional expectation $\E[ \thetastar_i | S ]$ is non-negative.
Hence it suffices to provide an upper bound on this quantity.
The lower bound in the complementary case $\prob\brackets{\thetastar_i = +1 | S} < \prob\brackets{\thetastar_i = -1 | S}$ can be derived analogously.

We have
\begin{align*}
  \E[ \thetastar_i | S ] \; &= \; \prob\brackets{\thetastar_i = +1 | S} - \prob\brackets{\thetastar_i = -1 | S} \\
                            &= \; \prob\brackets{\thetastar_i = -1 | S} \parens*{ \frac{ \prob\brackets{\thetastar_i = +1 | S}}{ \prob\brackets{\thetastar_i = -1 | S} } -1}  \\
                            &\leq \; \frac{1}{2} \parens*{ \frac{ \prob\brackets{\thetastar_i = +1 | S}}{ \prob\brackets{\thetastar_i = -1 | S} } -1} \numberthis \label{eq:bernlb_1}
\end{align*}
where we used the assumption $\prob\brackets{\thetastar_i = +1 | S} \geq \prob\brackets{\thetastar_i = -1 | S}$ (and hence $\prob\brackets{\thetastar_i = -1 | S}  \leq \sfrac{1}{2}$).

Next, we bound the ratio of probabilities by invoking Lemma~\ref{lem:bb_one_d} (note that we have $\tau \leq \sfrac{1}{4}$ and $n \leq \sfrac{1}{\tau^2}$ as required).
With probability $(1-\frac \gamma 2)$, $S$ is such that for all $i \in [d]$ we have
\[
  \frac{\Pr[\thetastar_i =+1 \mid S]}{\Pr[\thetastar_i = -1 \mid S]} \; \in \; \brackets*{ \exp\parens*{-15\tau\sqrt{2n\log \frac{4d}{\gamma}}}, \; \exp\parens*{15\tau\sqrt{2n\log \frac{4d}{\gamma}}}} \; .
\]
Substituting this into Equation \eqref{eq:bernlb_1} then yields
\begin{align*}
  \E[ \thetastar_i | S ] \; &\leq \; \frac{1}{2} \parens*{ \exp\parens*{15\tau\sqrt{2n\log \frac{4d}{\gamma}}} - 1}  \\
                            &\leq \; 15\tau\sqrt{2n\log \frac{4d}{\gamma}} 
\end{align*}
where we used the inequality $e^x - 1 \leq 2x$ for $0 \leq x \leq 1$ (note that the upper bound on $n$ in the theorem implies that the argument to the exponential function is in this range).

Combining the bound above with the analogous lower bound gives
\[
  \abs*{\E[\thetastar_i | S]} \; \leq \; 15\tau\sqrt{2n\log \frac{4d}{\gamma}}
\]
so that
\begin{align*}
  \E_{\thetastar}\brackets*{\ip{w, \thetastar} | S} \; & = \; \sum_{i=1}^d \E_{\thetastar}\brackets*{w_i \thetastar_i | S} \\
                                                       & = \; \sum_{i=1}^d w_i \cdot \E_{\thetastar}\brackets*{\thetastar_i | S} \\
                                                       & \leq \; 15\tau\sqrt{2n\log \frac{4d}{\gamma}} \cdot \norm{w}_1 \; .
\end{align*}
We condition on such an $S$ for the rest of this proof.

The second part of the proof will bound the classification margin the linear classifier $w$ achieves on a fresh sample $x$.
Incorporating the class label $y$, this margin is the quantity $y\ip{w, x}$.
From the first part of the proof, it follows that
\begin{align*}
  \Eop_{\thetastar, (x,y)}[ \ip{w, y x} ] \; &= \;  2 \tau \cdot \E_{\thetastar} [\ip{w, \thetastar}] \\
                                         &\leq \; 30\tau^2 \sqrt{2n\log (4d/\gamma)}\cdot \norm{w}_1 \; .
\end{align*}
To simplify the following calculation, we introduce the shorthand $a_n \eqdef 30\tau^2 \sqrt{2n\log (4d/\gamma)}$.
Next, we provide a tail bound on $\ip{w, yx}$.
Similar to Lemma \ref{lem:linear_besteps}, we observe that the random variable $y x_i w_i$ is stochastically dominated by $E_i \cdot \abs{w_i}$ where $E_i$ is a $\pm 1$ random variable with expectation $a_n$.
We can again write $E_i$ as
\[
  E \; = \; A_i + B_i \; ,
\]
where the random variable $A_i$ is in $\{0,1\}$ and has expectation $a_n$.
The random variable $B_i$ is in $\{-1, 0, 1\}$ and has a symmetric distribution that depends on $A$.
In particular, $B_i = 0$ iff $A_i = 1$ and $B_i$ is a Rademacher random variable otherwise.
Since $A_i$ is non-negative, we can use Markov's inequality on $\sum_i \abs{w_i} A_i$.
The $B_i$'s have a symmetric distribution even conditioned on $A_i$ so that $\sum_i \abs{w_i} B_i \leq 0$ with probability at least $\sfrac{1}{2}$.
Thus with probability at least $\frac{1-\gamma}{2}$, we have
\[
  \ip{w, y x} \; \leq \; \frac{a_n}{\gamma} \norm{w}_1 \; .
\]
Using the upper bound on $n$ from the theorem statement, we have
\begin{align*}
  \frac{a_n}{\gamma} \; &\leq \; \frac{30 \tau^2 \sqrt{ 2n \log(4d / \gamma)}}{\gamma} \\
                        &< \; \eps \; .
\end{align*}
Next, consider the strongest adversarial perturbation $e \in \perturbB_{\infty}^{\eps}$ for a given $w$, i.e., the vector $e \in \R^d$ achieving
\[
  \min_{e \in \perturbB_{\infty}^\eps} \ip{w, e} \; .
\]
By duality, the minimum value is exactly $\eps \norm{w}_{1}$.
Hence conditioned on the samples $S$ and the bound on $\ip{w, yx}$, the adversarially perturbed point $x + e$ is mis-classified because
\begin{align*}
  y\ip{w, x} \; &= \; y \ip{w, x} + y\ip{w, e} \\
  &< \; \eps \norm{w}_1 - \eps \norm{w}_1 \\
  &= \; 0 \; .
\end{align*}
The overall probability of this event occuring is at least $1 - \frac{\gamma}{2}$ (conditioning on $S$) times $\frac{1-\gamma}{2}$ (bound on $\ip{w, x}$).
Since
\[
  \parens*{1 - \frac{\gamma}{2}} \parens*{\frac{1 - \gamma}{2}} \; \geq \; \frac{1}{2} - \gamma \; .
\]
this completes the proof.
\end{proof}

\section{Omitted Figures}
\label{sec:additional_experiments}
\begin{figure*}[!htp]
\begin{center}
\includegraphics[width=0.9\textwidth]{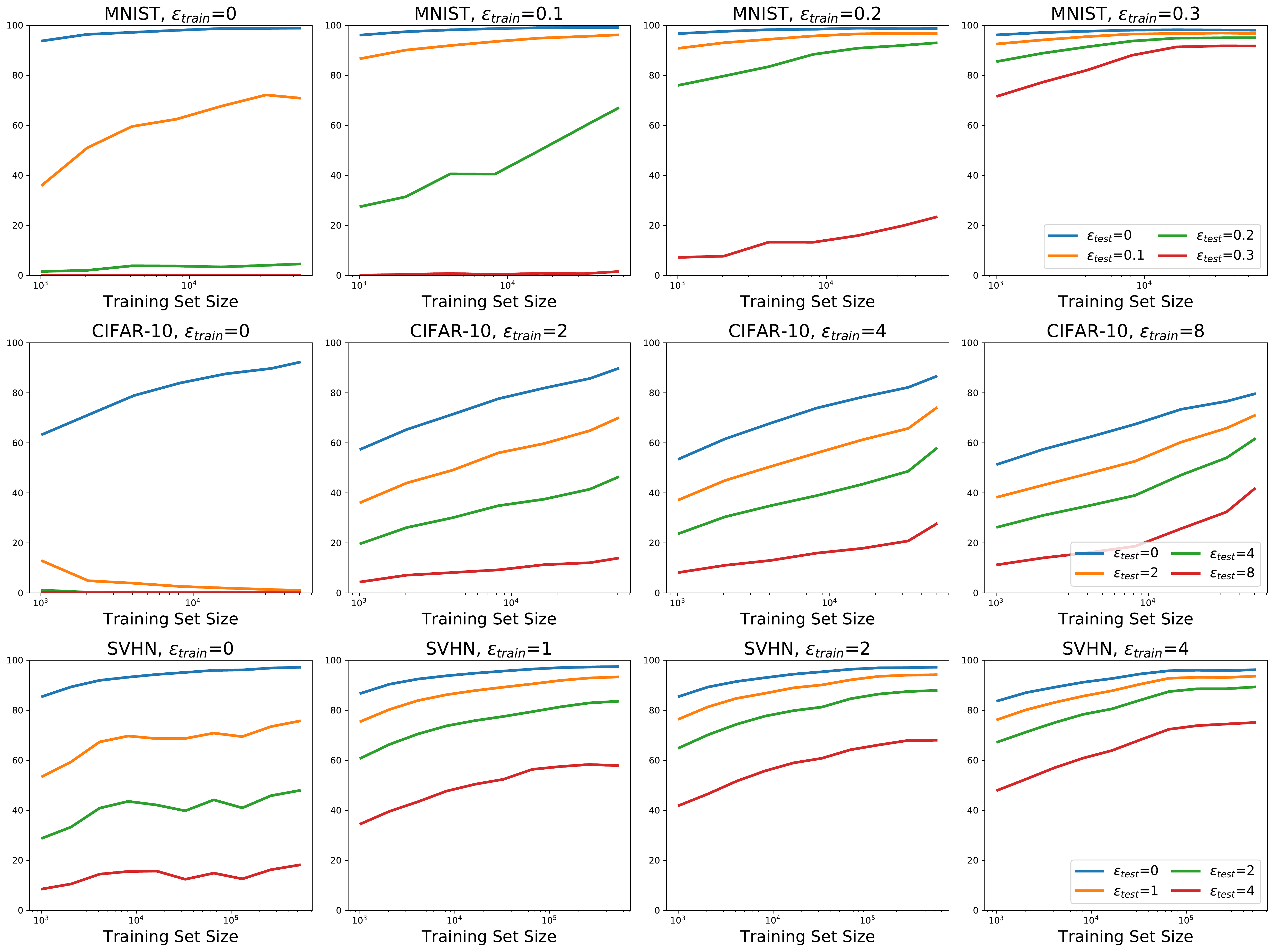}
\caption{Complete experiments for adversarially robust generalization for $\ell_\infty$ adversaries.
For each dataset and training $\eps$ we report the performance of the corresponding classifier for each testing $\eps$.
We observe that the best performance on natural examples is achieved through natural training and the best adversarial performance is achieved when training with the largest $\eps_{train}$ considered.}
\label{fig:linf_complete}
\end{center}
\end{figure*}

\begin{figure*}[!htp]
\begin{center}
\includegraphics[width=0.9\textwidth]{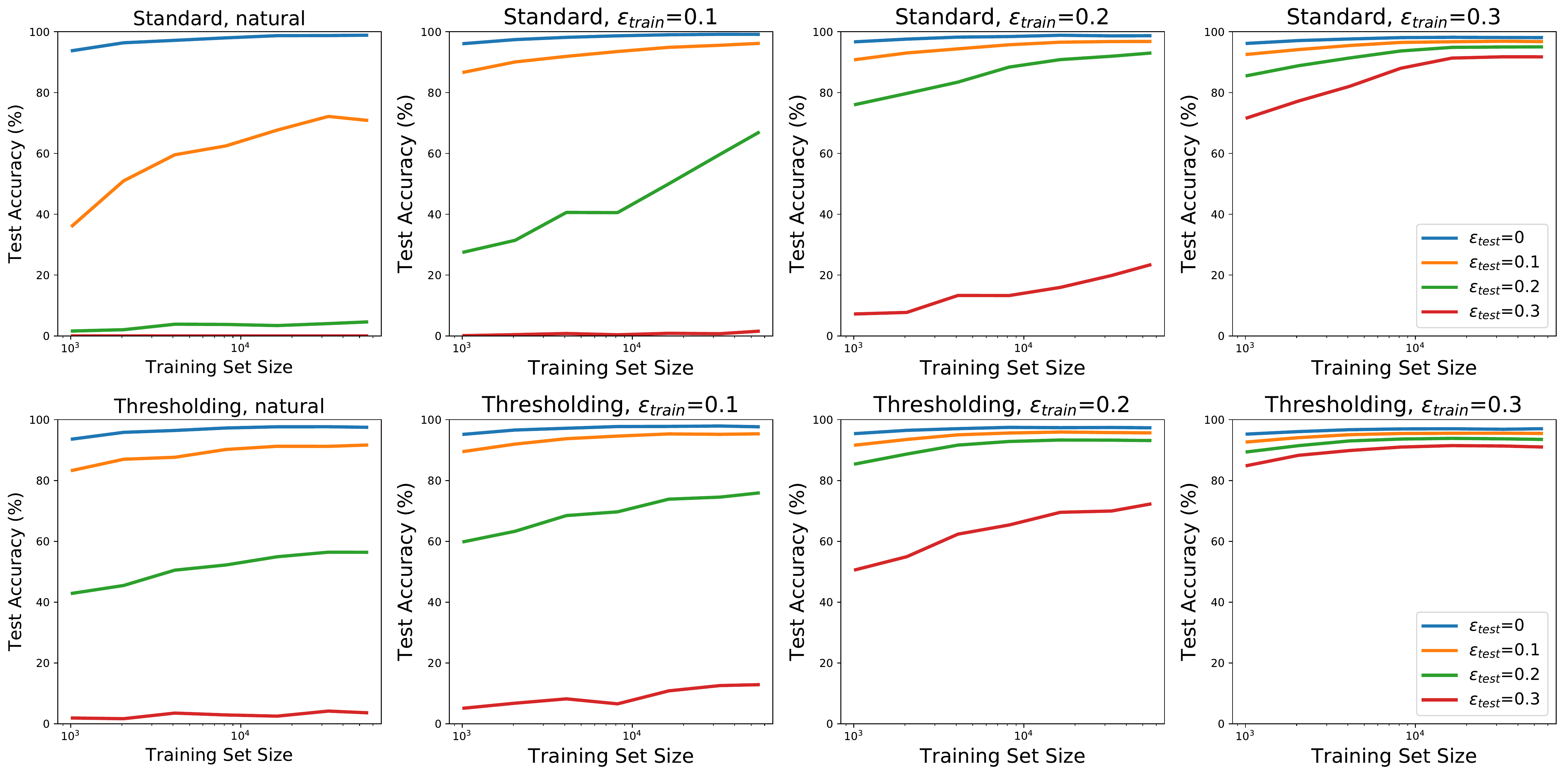}
\caption{Complete experiments for adversarially robust generalization for $\ell_\infty$ adversaries for standard networks (\textit{top row}) and networks with thresholding (\textit{bottom row}) for the MNIST dataset.
Thresholding corresponds to replacing the first convolutional layer with two channels $\textrm{ReLU}(x - \eps)$ computing $\textrm{ReLU}(x - (1-\eps))$.
For each training $\eps_{train}$ we report the performance of the corresponding classifier for each testing $\eps_{test}$.
For natural training, we use thresholding filters identical to those used for $\eps_{train}=0.1$.
We observe that in each case, explicitly encoding thresholding filters in the network architecture boosts the adversarial robustness for a given training $\eps_{train}$ and training set size.}
\label{fig:linf_thresh_complete}
\end{center}
\end{figure*}

\end{document}